%% file: DBCL.tex
\documentclass{article}

\usepackage{microtype}
\usepackage{graphicx}
\usepackage{subfigure}
\usepackage{booktabs} 
\usepackage{hyperref}       

\usepackage[accepted]{icml2021}

\usepackage{url}            
\usepackage{nicefrac}       
\usepackage{microtype}      
\usepackage{wrapfig}
\usepackage{amsfonts}       
\usepackage{amsmath}
\usepackage{amssymb}
\usepackage{amsthm}
\usepackage{color}
\usepackage{listings} 

\usepackage{multirow}
\usepackage{multicol}
\usepackage{ltxtable} 
\newcolumntype{L}{>{\centering\arraybackslash} m{0.04\columnwidth}} 
\newcolumntype{R}{>{\centering\arraybackslash} m{0.48\columnwidth}} 
\newcolumntype{S}{>{\centering\arraybackslash} m{0.32\columnwidth}} 

\newtheorem{theorem}{Theorem}
\newtheorem{lemma}[theorem]{Lemma}
\newtheorem{corollary}[theorem]{Corollary}

\newtheorem{remark}{Remark}

\def\A{{\bf A}}
\def\a{{\bf a}}
\def\B{{\bf B}}
\def\bb{{\bf b}}
\def\C{{\bf C}}

\def\e{{\bf e}}

\def\g{{\bf g}}

\def\G{{\bf G}}

\def\I{{\bf I}}
\def\K{{\bf K}}

\def\PP{{\bf P}}
\def\pp{{\bf p}}

\def\S{{\bf S}}

\def\V{{\bf V}}

\def\W{{\bf W}}
\def\w{{\bf w}}
\def\X{{\bf X}}
\def\x{{\bf x}}
\def\Y{{\bf Y}}

\def\Z{{\bf Z}}
\def\z{{\bf z}}
\def\0{{\bf 0}}
\def\1{{\bf 1}}

\def\OM{{\mathcal O}}

\def\RB{{\mathbb R}}
\def\EB{{\mathbb E}}

\def\De{{\boldsymbol \Delta}}

\def\Gam{\mbox{\boldmath$\Gamma$\unboldmath}}

\def\argmin{\mathop{\rm argmin}}

\def\vect{\mathrm{vec}}

\def\etal{{\em et al.\/}\,}

\def\din{{d_{\textrm{in}}}}
\def\dout{{d_{\textrm{out}}}}

\icmltitlerunning{Matrix Sketching for Secure Collaborative Machine Learning}

\author{
 Mengjiao Zhang {\texttt{and}} Shusen Wang \\
 Department of Computer Science\\
 Stevens Institute of Technology\\
 Hoboken, NJ 07030 \\
 \texttt{\{mzhang49, shusen.wang\}@stevens.edu} \\
}

\begin{document}

\twocolumn[
\icmltitle{Matrix Sketching for Secure Collaborative Machine Learning}



%
\begin{icmlauthorlist}
	\icmlauthor{Mengjiao Zhang}{stevens}
	\icmlauthor{Shusen Wang}{stevens}
\end{icmlauthorlist}
\icmlaffiliation{stevens}{ Department of Computer Science, Stevens Institute of Technology, Hoboken, NJ 07030}
%
\icmlcorrespondingauthor{Mengjiao Zhang}{mzhang49@stevens.edu}
\icmlcorrespondingauthor{Shusen Wang}{shusen.wang@stevens.edu}

\icmlkeywords{Federated Learning, Matrix Sketching}

\vskip 0.3in
]

\printAffiliationsAndNotice{}

\begin{abstract}
Collaborative learning allows participants to jointly train a model without data sharing. To update the model parameters, the central server broadcasts model parameters to the clients, and the clients send updating directions such as gradients to the server. While data do not leave a client device, the communicated gradients and parameters will leak a client's privacy. Attacks that infer clients' privacy from gradients and parameters have been developed by prior work. Simple defenses such as dropout and differential privacy either fail to defend the attacks or seriously hurt test accuracy.
We propose a practical defense which we call Double-Blind Collaborative Learning (DBCL). The high-level idea is to apply random matrix sketching to the parameters (aka weights) and re-generate random sketching after each iteration. DBCL prevents clients from conducting gradient-based privacy inferences which are the most effective attacks. DBCL works because from the attacker's perspective, sketching is effectively random noise that outweighs the signal. Notably, DBCL does not much increase computation and communication costs and does not hurt test accuracy at all. 
\end{abstract}

\section{Introduction}
\label{sec:intro}

Collaborative learning allows multiple parties to jointly train a model using their private data but without sharing the data.
Collaborative learning is motivated by real-world applications, for example, training a model using but without collecting mobile user's data.


Distributed stochastic gradient descent (SGD) 
is perhaps the simplest approach to collaborative learning.
Specifically, the central server broadcasts model parameters to the clients, each client uses a batch of local data to compute a stochastic gradient, and the server 
aggregates the stochastic gradients and updates the model parameters.
Based on distributed SGD, communication-efficient algorithms such as federated averaging (\texttt{FedAvg}) \cite{mcmahan2017communication} and \texttt{FedProx} \cite{sahu2019federated} have been developed and analyzed \cite{li2019convergence,stich2018local,wang2018cooperative,yu2019parallel,zhou2017convergence}.


Collaborative learning seemingly protects clients' privacy.
Unfortunately, this has been demonstrated not true by recent studies \cite{hitaj2017deep,melis2019exploiting,zhu2019deep}.
Even if a client's data do not leave his device, important properties of his data can be disclosed from the model parameters and gradients.
To infer other clients' data, the attacker needs only to control one client device and access the model parameters in every iteration; the attacker does not have to take control of the server \cite{hitaj2017deep,melis2019exploiting,zhu2019deep}.

The reason why the attacks work is that model parameters and gradients carry important information about the training data \cite{ateniese2015hacking,fredrikson2015model}.
\citet{hitaj2017deep} used the jointly learned model as a discriminator for training a generator which generates other clients' data.
\citet{melis2019exploiting} used gradient for inferring other clients' data properties.
\citet{zhu2019deep} used both model parameters and gradients for recovering other clients' data.
Judging from published empirical studies, the gradient-based attacks \cite{melis2019exploiting,zhu2019deep} are more effective than the parameter-based attack \cite{hitaj2017deep}.
Our goal is to defend the gradient-based attacks.

Simple defenses, e.g., differential privacy \cite{dwork2011differential} and dropout \cite{srivastava2014dropout}, have been demonstrated not working well by \cite{hitaj2017deep,melis2019exploiting}.
While differential privacy \cite{dwork2011differential}, i.e., adding noise to model parameters or gradients, works if the noise is strong, the noise inevitably hurts the accuracy and may even stop the collaborative learning from making progress \cite{hitaj2017deep}.
If the noise is not strong enough, clients' privacy will leak.
Dropout training \cite{srivastava2014dropout} randomly masks a fraction of the parameters, making the clients have access to only part of the parameters in each iteration.
However, knowing part of the parameters is sufficient for conducting the attacks.

\paragraph{Proposed method.}
We propose \textit{Double-Blind Collaborative Learning (DBCL)} as a practical defense against gradient-based attacks, e.g., \cite{melis2019exploiting,zhu2019deep}.
Technically speaking, DBCL applies random sketching \cite{woodruff2014sketching} to the parameter matrices of a neural network, and the random sketching matrices are regenerated after each iteration.
Throughout the training, the clients do not see the real model parameters; what the clients see are sketched parameters.
The server does not see any real gradient or descending direction; what the server sees are approximate gradients based on sketched data and sketched parameters.
This is why we call our method \textit{double-blind}.

From an honest client's perspective, DBCL is similar to dropout training, except that we use sketching to replace uniform sampling; see the discussions in Section~\ref{sec:theory:optimization}.
It is very well known that dropout does not hurt test accuracy at all \cite{srivastava2014dropout,wager2013dropout}.
Since sketching has similar properties as uniform sampling, DBCL does not hurt test accuracy, which is corroborated by our empirical studies.

From an attacker's perspective, DBCL is effectively random noise injected into the gradient, which the attacker needs for inferring other clients' privacy.
Roughly speaking, if a client tries to estimate the gradient, then he will get
\vspace{-1mm}
\begin{equation*}
    \textrm{Estimated Grad}
    \; = \; \textrm{Transform } \big( \textrm{True Grad} \big)
    \, + \, \textrm{Noise} .
\vspace{-1mm}
\end{equation*}
Therefore, client-side gradient-based attacks do not work.
Detailed discussions are in Section~\ref{sec:theory:defense}.

In addition to its better security, DBCL has the following nice features:
\vspace{-1mm}
\begin{itemize}
    \item 
    DBCL does not hinder test accuracy. This makes DBCL superior to the existing defenses.
\vspace{-1mm}
    \item
    DBCL does not increase the per-iteration time complexity and communication complexity, although it reasonably increases the iterations for attaining convergence.
\vspace{-1mm}
    \item
    When applied to dense layers and convolutional layers, DBCL does not need extra tuning.
\vspace{-1mm}
\end{itemize}
It is worth mentioning that DBCL is not an alternative to the existing defenses; instead, DBCL can be combined with the other defenses (except dropout).

\paragraph{Difference from prior work.}
Prior work has applied sketching to neural networks for privacy, computation, or communication benefits \cite{hanzely2018sega,li2019privacy}.
\citet{blocki2012johnson,kenthapadi2012privacy} showed that random sketching are differential private, but the papers have very different methods and applications from our work.
\citet{hanzely2018sega,li2019privacy} applies sketching to the gradients during federated learning.
In particular, \citet{li2019privacy} showed that sketching the gradients is differential private.

Our method is different from the aforementioned ones.
We sketch the model parameters, not the gradients.
We re-generate sketching matrices after each communication.
Note that sketching the gradients protects privacy at the cost of substantial worse test accuracy.
In contrast, our method does not hurt test accuracy at all.

\paragraph{Limitations.}
While we propose DBCL as a practical defense against gradient-based attacks at little cost, we do not claim DBCL as a panacea.
DBCL has two limitations.
First, with DBCL applied, a malicious client cannot perform gradient-based attacks, but he may be able to perform parameter-based attacks such as \cite{hitaj2017deep};
fortunately, the latter is much less effective than the former.
Second, DBCL cannot prevent a malicious server from inferring clients' privacy, although DBCL makes the server's attack much less effective.

We admit that DBCL alone does not defeat all the attacks.
To the best of our knowledge, there does not exist any defense that is effective for all the attacks that infer privacy.
DBCL can be easily incorporated with existing methods such as homomorphic encryption and secret sharing to defend server-side attacks.

\paragraph{Paper organization.}
Section~\ref{sec:pre} introduces neural network, backpropagation, and matrix sketching.
Section~\ref{sec:attack} defines threat models.
Section~\ref{sec:alg} describes the algorithm, including the computation and communication.
Section~\ref{sec:theory} theoretically studies DBCL.
Section~\ref{sec:exp} presents empirical results to demonstrate that DBCL does not harm test accuracy, does not much increase the communication cost, and can defend gradient-based attacks.
Section~\ref{sec:related} discusses related work.
Algorithm derivations and theoretical proofs are in the appendix.
The source code is available at the Github repo: 
\url{https://github.com/MengjiaoZhang/DBCL}

\section{Preliminaries} \label{sec:pre}

\paragraph{Dense layer.}
Let $\din$ be the input shape, $\dout$ be the output shape, and $b$ be the batch size.
Let $\X \in \RB^{b\times \din}$ be the input, $\W \in \RB^{\dout \times \din}$ be the parameter matrix, and $\Z = \X \W^T$ be the output.
After the dense (aka fully-connected) layer, there is typically an activation function $\sigma (\Z)$ applied elementwisely.

\paragraph{Backpropagation.}
Let $L$ be the loss evaluated on a batch of $b$ training samples.
We derive backpropagation for the dense layer by following the convention of PyTorch.
Let $\G \triangleq \frac{ \partial \, L }{\partial \, \Z} \in \RB^{b\times \dout}$ be the gradient received from the upper layer.
We need to compute the gradients:
\vspace{-1mm}
\begin{small}
	\begin{equation*}
	\frac{ \partial \, L }{\partial \, \X}
	\: = \: \G \W 
	\: \in \: \RB^{ b \times \din }
	\quad \textrm{and} \quad
	\frac{ \partial \, L }{\partial \, \W}
	\: = \: \G^T \X
	\: \in \: \RB^{ \dout \times \din } ,
	\end{equation*}
\end{small}
which can be established by the chain rule.
We use $\frac{ \partial \, L }{\partial \, \W}$ to update the parameter matrix $\W$ by e.g., $\W \leftarrow \W - \eta \frac{ \partial \, L }{\partial \, \W}$, and pass $\frac{ \partial \, L }{\partial \, \X}$ to the lower layer.

\paragraph{Uniform sampling matrix.}
We call $\S \in \RB^{\din \times s}$ a uniform sampling matrix if its columns are sampled from the set $\big\{ \tfrac{\sqrt{\din}}{\sqrt{s}} \e_1, \cdots , \tfrac{\sqrt{\din}}{\sqrt{s}} \e_{\din} \big\}$
uniformly at random.
Here, $\e_i$ is the $i$-th standard basis of $\RB^{\din}$.
We call $\S$ a uniform sampling matrix because $\X \S$ contains $s$ randomly sampled (and scaled) columns of $\X$.
Random matrix theories \cite{drineas2016randnla,mahoney2011ramdomized,martinsson2020randomized,woodruff2014sketching} guarantee that $\EB_{\S} \big[ \X \S \S^T \W^T \big] = \X \W^T $ and that $\| \X \S \S^T \W^T - \X \W^T \|$ is bounded, for any $\X$ and $\W$.

\paragraph{CountSketch.}
We call $\S \in \RB^{\din \times s}$ a CountSketch matrix \cite{charikar2004finding,clarkson2013low,pham2013fast,thorup2012tabulation,weinberger2009feature} if it is constructed in the following way.
Every row of $\S$ has exactly one nonzero entry whose position is randomly sampled from $[s] \triangleq \{1, 2, \cdots , s\}$ and value is sampled from $\{-1, +1\}$.
Here is an example of $\S$ ($10 \times 3$):
\begin{footnotesize}
	\[
	\S^T \; = \;
	\left[
	\begin{array}{cccccccccccccccc}
	0 & 0 & 1 &-1 & 1 &-1 & 0 & 0 & 0 & 0 \\
	-1& 0 & 0 & 0 & 0 & 0 & 1 & 1 &-1 & 0 \\
	0 &-1 & 0 & 0 & 0 & 0 & 0 & 0 & 0 & 1 \\
	\end{array}
	\right].
	\]%
\end{footnotesize}%
CountSketch has very similar properties as random Gaussian matrices \cite{johnson1984extensions,woodruff2014sketching}.
We use CountSketch for its computation efficiency.
Given $\X \in \RB^{b\times \din}$, the CountSketch $\widetilde{\X} = \X \S$ can be computed in $\OM (\din b)$ time.
CountSketch is much faster than the standard matrix multiplication which has $\OM (\din b s)$ time complexity.
Theories in \cite{clarkson2013low,meng2013low,nelson2013osnap,woodruff2014sketching} guarantee that $\EB_{\S} \big[ \X \S \S^T \W^T \big] = \X \W^T $ and that $\| \X \S \S^T \W^T - \X \W^T \|$ is bounded, for any $\X$ and $\W$.
In practice, $\S$ is not explicitly constructed.

\section{Threat Models} \label{sec:attack}

In this paper, we consider attacks and defenses under the setting of client-server architecture; we assume the attacker controls a client.\footnote{A stronger assumption would be that the server is malicious. Our defense may not defeat a malicious server.}
Let $\W_{\textrm{old}}$ and $\W_{\textrm{new}}$ be the model parameters (aka weights) in two consecutive iterations.
The server broadcasts $\W_{\textrm{old}}$ to the clients, the $m$ clients use $\W_{\textrm{old}}$ and their local data to compute ascending directions $\De_1, \cdots , \De_m$ (e.g., gradients), and the server aggregates the directions by $\De=\frac{1}{m} \sum_{i=1}^m \De_i$ and performs the update 
$\W_{\textrm{new}} \leftarrow \W_{\textrm{old}} - \De$.
Since a client (say the $k$-th) knows $\W_{\textrm{old}}$, $\W_{\textrm{new}}$, and his own direction $\De_k$, he can calculate the sum of other clients' directions by
\begin{eqnarray} \label{eq:client_infer} 
\sum_{i\neq k} \De_i 
& = & m \De - \De_k \nonumber \\
& = & m \big(\W_{\textrm{old}} - \W_{\textrm{new}} \big) - \De_k .
\end{eqnarray}
In the case of two-party collaborative learning, that is, $m=2$, one client knows the updating direction of the other client.

Knowing the model parameters, gradients, or both, the attacker can use various ways \cite{hitaj2017deep,melis2019exploiting,zhu2019deep} to infer other clients’ privacy.
We focus on gradient-based attacks \cite{melis2019exploiting,zhu2019deep}, that is, the victim's privacy is extracted from the gradients.
Melis \etal (2019) \cite{melis2019exploiting} built a classifier and locally trained it for property inference.
The classifier takes the updating direction $\De_i$ as input feature and predicts the clients' data properties.
The client's data cannot be recovered, however, the classifier can tell, e.g., the photo is likely female.
Zhu \etal (2019) \cite{zhu2019deep} developed an optimization method called \textit{gradient matching} for recovering other clients' data; both gradient and model parameters are used.
It has been shown that simple defenses such as differential privacy \cite{dwork2010difficulties,dwork2011differential} and dropout \cite{srivastava2011distributed} cannot defend the attacks.

In decentralized learning, where participants are compute nodes in a peer-to-peer network, a node knows its neighbors' model parameters and thereby updating directions.
A node can infer the privacy of its neighbors in the same way as \cite{melis2019exploiting,zhu2019deep}.
In Appendix~\ref{sec:decentralized}, we discuss the attack and defense in decentralized learning; they will be our future work.

\section{Proposed Method}\label{sec:alg}

We present the high-level ideas in Section~\ref{sec:alg:idea}, elaborate on the implementation in Section~\ref{sec:alg:description}, and analyze the time and communication complexities in Section~\ref{sec:alg:complexity}.

\subsection{High-Level Ideas} \label{sec:alg:idea}

The attacks of \cite{melis2019exploiting,zhu2019deep} need the victim's updating direction, e.g., gradient, for inferring the victim's privacy.
Using standard distributed algorithms such as distributed SGD and Federated Averaging (\texttt{FedAvg}) \cite{mcmahan2017communication}, the server can see the clients' updating directions, $\De_1 , \cdots, \De_m$, and the clients can see the jointly learned model parameter, $\W$.
A malicious client can use \eqref{eq:client_infer} to get other clients' updating directions and then perform the gradient-based attacks such as \cite{melis2019exploiting,zhu2019deep}.

To defend the gradient-based attacks, our proposed Double-Blind Collaborative Learning (DBCL) applies random sketching to the inputs and parameter matrices.
Let $\X$ and $\W$ be the input batch and parameter matrix of a dense layer, respectively.
For some or all the dense layers, replace $\X$ and $\W$ by $\widetilde{\X} = \X \S$ and $\widetilde{\W} = \W \S$, respectively, and different layers have different sketching matrices, $\S$.
Each time $\W$ is updated, we re-generate $\S$.
DBCL is applicable to convolutional layers in a similar way; see Appendix~\ref{sec:conv}.




Let $\W_{\textrm{old}}$ and $\W_{\textrm{new}} $ be the true parameter matrices in two consecutive iterations; they are known to only the server.
The clients do not observe $\W_{\textrm{old}}$ and $\W_{\textrm{new}} $.
What the clients observe are the random sketches: $\widetilde{\W}_{\textrm{old}}  = \W_{\textrm{old}}  \S_{\textrm{old}} $ and $\widetilde{\W}_{\textrm{new}}  = \W_{\textrm{new}}  \S_{\textrm{new}} $.
In addition, the clients know $\S_{\textrm{old}}$ and $\S_{\textrm{new}}$.
To perform gradient-based attacks, a client seeks to estimate $\De = \W_{\textrm{old}} - \W_{\textrm{new}}$ based on what it knows, i.e., $\widetilde{\W}_{\textrm{old}} $, $\widetilde{\W}_{\textrm{new}} $, $\S_{\textrm{old}}$ and $\S_{\textrm{new}}$.
Our theories and experiments show that a client's estimate of $\De$ is far from the truth.


\subsection{Algorithm Description} \label{sec:alg:description}

We describe the computation and communication operations of DBCL.
We consider the client-server architecture, dense layers, and the distributed SGD algorithm.\footnote{DBCL works also for convolutional layers; see the appendix for the details.
DBCL can be easily extended to \texttt{FedAvg} or other communication-efficient frameworks.
DBCL can be applied to peer-to-peer networks; see the discussions in the appendix.}
DBCL works in the following four steps.
Broadcasting and aggregation are communication operations; forward pass and backward pass are local computations performed by each client for calculating gradients.

\paragraph{Broadcasting.}
The central server generates a new seed $\psi$\footnote{Let the clients use the same pseudo-random number generator as the server. Given the seed $\psi$, all the clients can construct the same sketching matrix $\S$.} and then a random sketch:
$\widetilde{\W}  = \W  \S $.
It broadcasts $\psi$ and $\widetilde{\W} \in \RB^{\dout \times s} $ to all the clients through message passing.
Here, the sketch size $s$ is determined by the server and must be set smaller than $\din$; the server may vary $s$ after each iteration.

\paragraph{Local forward pass.}
The $i$-th client randomly selects a batch of $b$ samples from its local dataset and then locally performs a forward pass.
Let the input of a dense layer be $\X_i  \in \RB^{b\times \din} $.
The client uses the seed $\psi$ to draw a sketch $\widetilde{\X}_i = \X_i \S \in \RB^{b\times s} $ and computes $\Z_i = \widetilde{\X}_i \widetilde{\W}^T$.
Then $\sigma (\Z_i )$ becomes the input of the upper layer, where $\sigma$ is some activation function.
Repeat this process for all the layers.
The forward pass finally outputs $L_i$, the loss evaluated on the batch of $b$ samples.

\paragraph{Local backward pass.}
Let the local gradient propagated to the dense layer be $\G_i \triangleq \frac{ \partial \, L_i }{ \partial \, \Z_i } \in \RB^{b \times \dout} $.
The client locally calculates
\begin{small}
\begin{align*}
& \Gam_{i}
\: = \: \G_i^T \widetilde{\X}_i 
\: \in \: \RB^{\dout \times s} 
\;\; \textrm{and} \;\;
\frac{ \partial \, L_i }{\partial \, \X_i}
\: = \: \G_i \widetilde{\W}  \S^T 
\: \in \: \RB^{b \times \din} .
\end{align*}
\end{small}%
The gradient $\frac{ \partial \, L_i }{\partial \, \X_i}$ is propagated to the lower-level layer to continue the backpropagation.

\paragraph{Aggregation.}
The server aggregates $\{ \Gam_{i} \}_{i=1}^m$ to compute $\Gam = \frac{1}{m} \sum_{i=1}^m \Gam_i$; this needs a communication.
Let $L = \frac{1}{m} \sum_{i=1}^m  L_i $ be the loss evaluated on the batch of $mb$ samples.
It can be shown that
\begin{small}
\begin{equation}\label{eq:grad_sketch}
\frac{ \partial \, L }{\partial \, \W}
\: = \: \frac{1}{m}  \sum_{i=1}^m \frac{ \partial \, L_i }{\partial \, \W}
\: = \: \Gam \S^T 
\: \in \: \RB^{\dout \times \din} .
\end{equation}
\end{small}%
The server then updates the parameters by, e.g., $\W \leftarrow \W - \eta \frac{ \partial \, L }{\partial \, \W}$.

\subsection{Time Complexity and Communication Complexity}\label{sec:alg:complexity}

DBCL does not increase the time complexity of local computations.
The CountSketch, $\widetilde{\X}_i = \X_i \S $ and $\widetilde{\W}_i = \W_i \S $, costs $\OM (b \din)$ and $\OM (\din \dout)$ time, respectively.
Using CountSketch, the overall time complexity of a forward and a backward pass is $\OM (b\din + \din \dout + b s \dout)$.
Since we set $s < \din$ to protect privacy, the time complexity is lower than the standard backpropagation, $\OM (b \din \dout)$.

DBCL does not increase per-iteration communication complexity.
Without using sketching, the communicated matrices are $\W \in \RB^{\dout \times \din}$ and $\frac{\partial \, L_i}{\partial \, \W} \in \RB^{\dout \times \din}$.
Using sketching, the communicated matrices are $\widetilde{\W} \in \RB^{\dout \times s}$ and $\Gam_i \in \RB^{\dout \times s}$.
Because  $s < \din$, the per-iteration communication complexity is lower than the standard distributed SGD.

\section{Theoretical Insights} \label{sec:theory}

In Section~\ref{sec:theory:threat}, we discuss how a malicious client makes use of the sketched model parameters for privacy inference.
In Sections~\ref{sec:theory:defense}, we show that DBCL can defend certain types of attacks.
In Section~\ref{sec:theory:optimization}, we give an explanation of DBCL from optimization perspective.

\subsection{Approximating the Gradient} \label{sec:theory:threat}

Assume the attacker controls a client and participates in collaborative learning.
Let $\W_{\textrm{old}}$ and $\W_{\textrm{new}}$ be the parameter matrices of two consecutive iterations;
they are unknown to the clients.
What a client sees are the sketches, $\widetilde{\W}_{\textrm{old}} =  \W_{\textrm{old}} \S_{\textrm{old}}$ and $\widetilde{\W}_{\textrm{new}} =  \W_{\textrm{new}} \S_{\textrm{new}}$.
To conduct gradient-based attacks, the attacker must know the gradient $\De = \W_{\textrm{old}} - \W_{\textrm{new}}$.

Without using $\S_{\textrm{old}}$ and $\S_{\textrm{new}}$, one cannot recover $\De$, not even approximately.
The difference between sketched parameters, $\widetilde{\De} = \widetilde{\W}_{\textrm{old}} - \widetilde{\W}_{\textrm{new}}$, is entirely different from the real gradient, ${\De} = {\W}_{\textrm{old}} - {\W}_{\textrm{new}}$.\footnote{The columns of $\S_{\textrm{old}}$ and $\S_{\textrm{new}}$ are randomly permuted. Even if $\widetilde{\De}$ is close to $\De$, after randomly permuting the columns of $\S_{\textrm{old}}$ or $\S_{\textrm{new}}$, $\widetilde{\De}$ becomes entirely different.}
We can even vary the sketch size $s$ with iterations so that $\widetilde{\W}_{\textrm{old}}$ and $\widetilde{\W}_{\textrm{new}}$ have different number of columns, making it impossible to compute ${\W}_{\textrm{old}} - {\W}_{\textrm{new}}$.

Note that the clients know also $\S_{\textrm{old}}$ and $\S_{\textrm{new}}$.
A smart attacker, who knows random matrix theory, may want to use
\[
\widehat{\De} = \W_{\textrm{old}} \S_{\textrm{old}} \S_{\textrm{old}}^T - \W_{\textrm{new}} \S_{\textrm{new}} \S_{\textrm{new}}^T 
\]
to approximate $\De$.
It is because $\widehat{\De}$ is an unbiased estimate of $\De$, i.e., $\EB \big[ \widehat{\De} \big] = \De$,
where the expectation is taken w.r.t.\ the random sketching matrices $\S_{\textrm{old}}$ and $\S_{\textrm{new}}$.

\subsection{Defending Gradient-Based Attacks} \label{sec:theory:defense}

We analyze the attack that uses $\widehat{\De}$.
We first give an intuitive explanation and then prove that using $\widehat{\De}$ does not work, unless the magnitude of $\De$ is smaller than $\W_{\textrm{new}}$.

\paragraph{Matrix sketching as implicit noise.}
As $\widehat{\De}$ is an unbiased estimate of $\De$, the reader may wonder why $\widehat{\De}$ does not disclose the information of $\De$.
Here we give an intuitive explanation.
$\widehat{\De}$ is a mix of $\De$ (which is the signal) and a random transformation of $\W_{\textrm{new}}$ (which is random noise):
\begin{small}
	\begin{equation} \label{eq:implicit_noise}
	\widehat{\De} 
	\: = \: \underbrace{\De }_{\textrm{signal}} \S_{\textrm{old}} \S_{\textrm{old}}^T + \underbrace{\W_{\textrm{new}} \big( \S_{\textrm{old}} \S_{\textrm{old}}^T - \S_{\textrm{new}} \S_{\textrm{new}}^T \big)}_{\textrm{zero-mean noise}} .
	\end{equation}
\end{small}%
As the magnitude of $\W$ is much greater than $\De$,\footnote{In machine learning, $\De$ is the updating direction, e.g., gradient. The magnitude of gradient is much smaller than the model parameters $\W$, especially when $\W$ is close to a stationary point.} the noise outweighs the signal, making $\widehat{\De}$ far from $\De$.
From the attacker's perspective, random sketching is effectively random noise that outweighs the signal.

\paragraph{Defending the gradient matching attack of \cite{zhu2019deep}.} 
The gradient matching attack can recover the victim's original data based on the victim's gradient, $\De_i$, and the model parameters, $\W$.
Numerical optimization is used to find the data on which the computed gradient matches $\De_i$.
To get $\De_i$, the attacker must know $\De = \sum_{i=1}^m \De_i$.
Using DBCL, no client knows $\De$.
A smart attacker may want to use $\widehat{\De}$ in lieu of $\De$ because of its unbiasedness.
We show in Theorem~\ref{thm:privacy2} that this approach does not work.

\vspace{1mm}

\begin{theorem} \label{thm:privacy2}
	Let $\S_{\textrm{old}} $ and $\S_{\textrm{new}}$ be ${\din\times s}$ CountSketch matrices and $s < \din$. 
	Then
	\begin{align*}
	\EB \big\| \widehat{\De}  - \De \big\|_F^2
	\: = \: \Omega \Big( \frac{\din}{s} \Big) \cdot \Big(  \big\| \W_{\textrm{old}}  \big\|_F^2 +  \big\| \W_{\textrm{new}} \big\|_F^2  \Big) .
	\end{align*}
\end{theorem}

Since the magnitude of $\De$ is typically smaller than $\W$, Theorem~\ref{thm:privacy2} guarantees that using $\widehat{\De}$ is no better than all-zeros or random guessing.

\begin{remark}
The theorem shows that DBCL beats one way of gradient estimate.
Admittedly, the theorem does not guarantee that DBCL can defend all kinds of gradient estimates.
A stronger bound would be:
\begin{equation*}
    \min_{\Z_1, \Z_2}
    \left\| \W_{\textrm{old}}  \S_{\textrm{old}} \Z_1 - \W_{\textrm{new}}  \S_{\textrm{new}} \Z_2 - \De \right\|_F^2
    = \Omega (1) \cdot \big\| \De \big\|_F^2 .
\end{equation*}
Unfortunately, at this time, we do not know how to prove such a bound.
\end{remark}

\paragraph{Defending the property inference attack (PIA) of \cite{melis2019exploiting}.} 
To conduct the PIA of \cite{melis2019exploiting}, the attacker may want to use a linear model parameterized by $\V$.\footnote{The conclusion applies also to neural networks because its first layer is such a linear model.}
According to \eqref{eq:client_infer}, the attacker uses $\De - \A$ as input features for PIA, where $\A$ is some fixed matrix known to the attacker.
The linear model makes prediction by $\Y \triangleq (\De-\A) \V^T$.
Using $\widehat{\De}$ to approximate $\De$, the prediction is $\widehat{\Y} \triangleq ( \widehat{\De} -\A) \V^T$.
Theorem~\ref{thm:privacy} and Corollary~\ref{cor:privacy} show that 
$\| \widehat{\Y} - \Y \|_F^2 = \| \widehat{\De}  \V^T - \De \V^T  \|_F^2 $ is very big. 

\vspace{2mm}

\begin{theorem} \label{thm:privacy}
	Let $\S_{\textrm{old}} $ and $\S_{\textrm{new}}$ be ${\din\times s}$ CountSketch matrices and $s < \din$. 
	Let $w_{pq}$ be the $(p, q)$-th entry of $\W_{\textrm{old}} \in \RB^{\dout \times \din}$
	and $\tilde{w}_{pq}$ be the $(p, q)$-th entry of $\W_{\textrm{new}} \in \RB^{\dout \times \din}$.
	Let $\V$ be any $r\times \din$ matrix and $v_{pq}$ be the $(p, q)$-th entry of $\V$.
	Then
	\begin{small}
	\begin{align*}
	& \EB \big\| \widehat{\De}  \V^T - \De \V^T \big\|_F^2
	\; = \: \frac{1}{s} \sum_{i=1}^{\dout} \sum_{j=1}^{r} \sum_{k\neq l}\\
	& \qquad \quad \bigg(  w_{ik}^2 v_{jl}^2 +  w_{ik} v_{jk} w_{il} v_{jl} + \tilde{w}_{ik}^2 v_{jl}^2 +  \tilde{w}_{ik} v_{jk} \tilde{w}_{il} v_{jl} \bigg) .
	\end{align*}
	\end{small}%
\end{theorem}

The bound in Theorem~\ref{thm:privacy} is involved.
To interpret the bound, we add (somehow unrealistic) assumptions and obtain Corollary~\ref{cor:privacy}.

\begin{corollary}\label{cor:privacy}
	Let $\S $ be a ${\din \times s}$ CountSketch matrix and $s < \din$. 
	Assume that the entries of $\W_{\textrm{old}}$ are IID and that the entries of $\V$ are also IID.
	Then
	\begin{align*}
	\EB \big\| \widehat{\De}  \V^T - \De \V^T \big\|_F^2
	\: = \: \Omega \Big( \frac{\din }{s}  \Big) \cdot \big\| \W_{\textrm{old}} \V^T  \big\|_F^2  .
	\end{align*}
\end{corollary}

Since the magnitude of $\De$ is much smaller than $\W$, especially when $\W$ is close to a stationary point,  $ \| \W \V^T \|_F^2$ is typically greater than $\| \De \V^T \|_F^2$.
Thus, $\EB \| \widehat{\De}  \V^T - \De \V^T \|_F^2$ is typically bigger than $\| \De \V^T \|_F^2$, which implies that using $\widehat{\De}$ is no better than all-zeros or random guessing.

\begin{table*}[h]
	\centering
	\caption {\label{tab:ACC_sketch} Experiments on MNIST.
		The table shows the rounds of communications for attaining the test accuracy.
		Here, $c$ is the participation ratio of \texttt{FedAvg}, that is, in each round, only a fraction of clients participate in the training.} 
	\vspace{2mm}
	\begin{small}
	\begin{tabular}{>{\centering\arraybackslash}p{0.16\linewidth}>{\centering\arraybackslash}p{0.1\linewidth}>{\centering\arraybackslash}p{0.1\linewidth}>{\centering\arraybackslash}p{0.1\linewidth}>{\centering\arraybackslash}p{0.1\linewidth}>{\centering\arraybackslash}p{0.1\linewidth}>{\centering\arraybackslash}p{0.12\linewidth}}
		\toprule
		\multirow{2}{*}{Models}& \multirow{2}{*}{Accuracy} & \multicolumn{5}{c}{Communication Rounds} \\
		& &$c = 1\%$ &$c = 10\%$ &$c=20\%$ &$c=50\%$ &$c = 100\%$\\\midrule
		MLP & $0.97$  & 222 & 96 & 84 & 83 & 82 \\
		MLP-Sketch  & $0.97$ & 572 & 322 & 308 & 298 & 287 \\  
		CNN & $0.99$ & 462 & 309 & 97 & 91 & 31\\
		CNN-Sketch & $0.99$ & 636 & 176 & 189 & 170 & 174\\\bottomrule
	\end{tabular}
	\end{small}
\end{table*}

\subsection{Understanding DBCL from Optimization Perspective} \label{sec:theory:optimization}

We give an explanation of DBCL from optimization perspective.
Let us consider the generalized linear model:
\begin{small}
	\begin{equation} \label{eq:optimization2}
	\argmin_{\w} \; \bigg\{ {f} (\w) \: \triangleq \: \frac{1}{n} \sum_{j=1}^n \ell \big( \x_i^T \w , \, y_j\big)  \bigg\} ,
	\end{equation}
\end{small}%
where $(\x_1 , y_1) , \cdots , (\x_n , y_n)$ are the training samples and $\ell(\cdot , \cdot)$ is the loss function.
If we apply sketching to a generalized linear model, then the training will be solving the following problem:
\begin{small}
\begin{equation} \label{eq:optimization1}
\argmin_{\w} \; \bigg\{ \tilde{f} (\w) \: \triangleq \: \EB_{\S} \bigg[ \frac{1}{n} \sum_{j=1}^n \ell \big( \x_i^T \S \S^T \w , \, y_j\big) \bigg] \bigg\}.
\end{equation}
\end{small}%
Note that \eqref{eq:optimization1} is different from \eqref{eq:optimization2}.
If $\S$ is a uniform sampling matrix, then \eqref{eq:optimization1} will be empirical risk minimization with dropout.
Prior work \cite{wager2013dropout} proved that dropout is equivalent to adaptive regularization which can alleviate overfitting.
Random projections such as CountSketch have the same properties as uniform sampling \cite{woodruff2014sketching}, and thus the role of random sketching in \eqref{eq:optimization1} can be thought of as adaptive regularization.
This is why DBCL does not hinder prediction accuracy at all.

\section{Experiments} \label{sec:exp}

We conduct experiments to demonstrate that first, DBCL does not harm test accuracy, second, DBCL does not increase the communication cost too much, and third, DBCL can defend client-side gradient-based attacks.

\subsection{Experiment Setting} \label{sec:exp:setting}

Our method and the compared methods are implemented using PyTorch.
The experiments are conducted on a server with 4 NVIDIA GeForce Titan V GPUs, 2 Xeon Gold 6134 CPUs, and 192 GB memory.
We follow the settings of the relevant papers to perform comparisons.

Three datasets are used in the experiments.
MNIST has 60,000 training images and 10,000 test images; each image is $28\times 28$.
CIFAR-10 has 50,000 training images and 10,000 test images; each image is $32\times 32 \times 3$.
Labeled Faces In the Wild (LFW) has 13,233 faces of 5,749 individuals; each face is a $64 \times 47 \times 3$ color image. 
\vspace{-2mm}

\begin{figure*}[htbp]
  \mbox{}\hfill
  \begin{minipage}[t]{.64\linewidth}
    \centering
    \subfigure[Test loss.]{\includegraphics[width=0.48\textwidth]{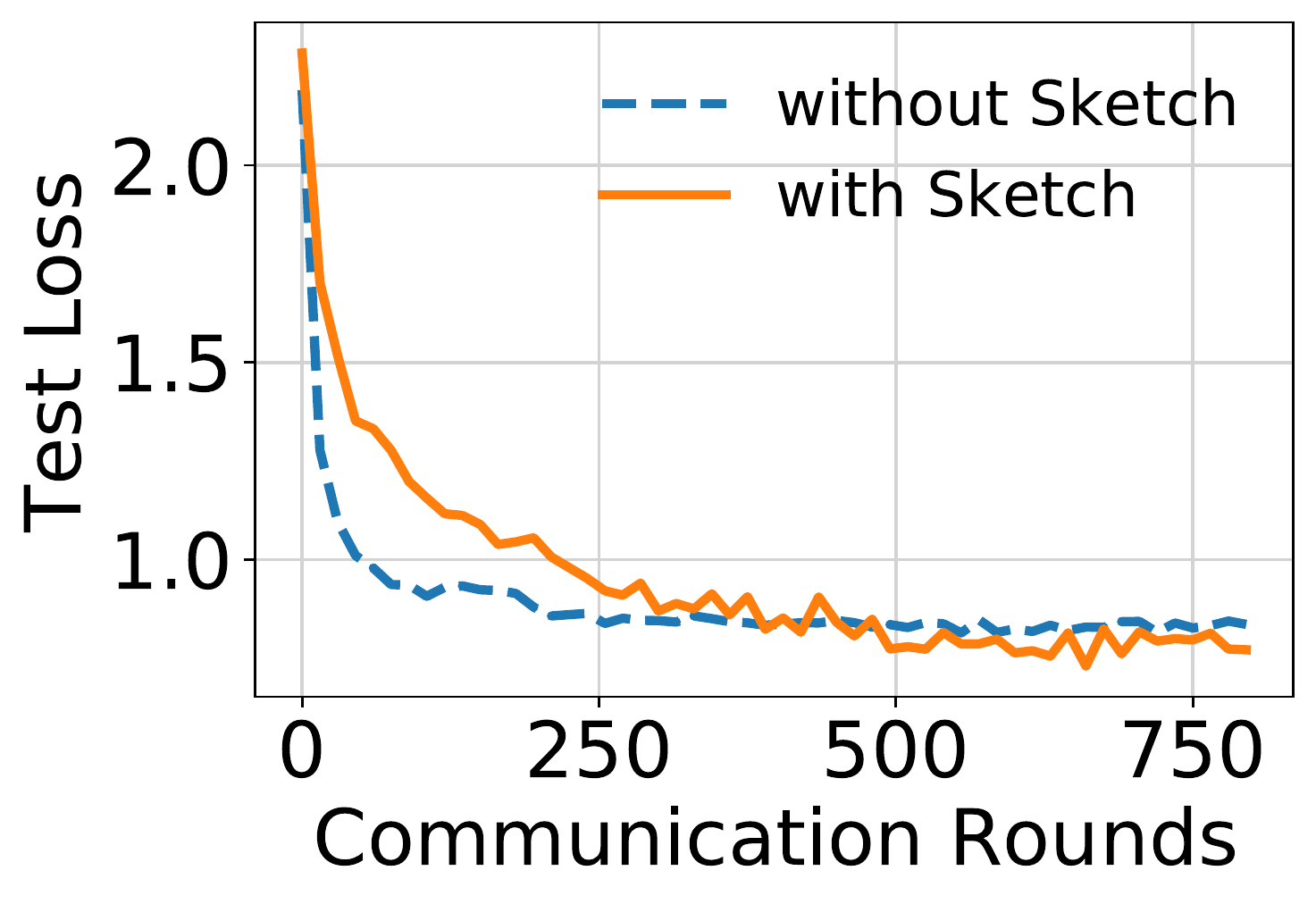}}
    ~~~
    \subfigure[Test Accuracy.]{\includegraphics[width=0.48\textwidth]{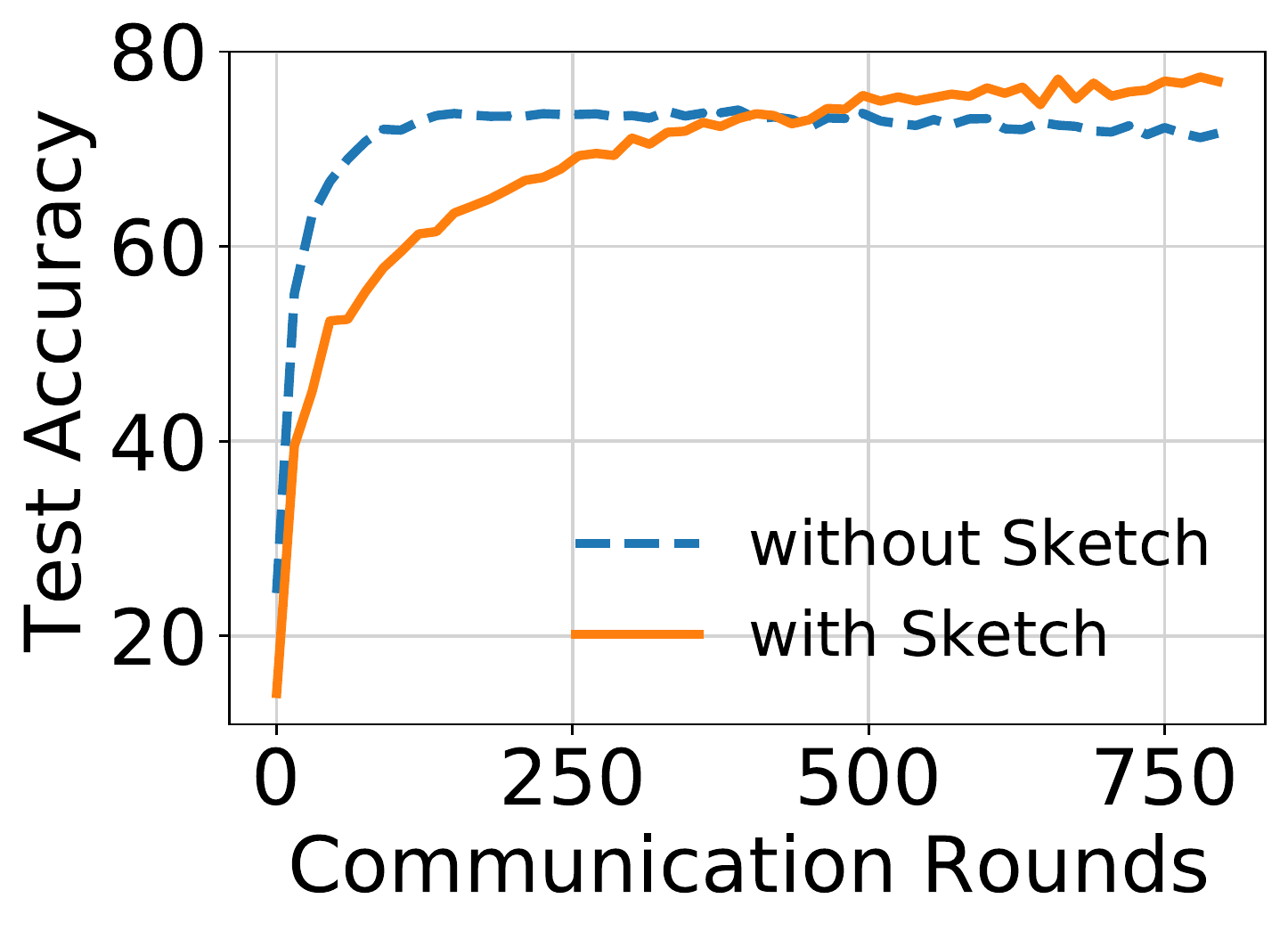}}
    \vspace{-3mm}
    \caption{Experiment on CIFAR-10 dataset. 
		The test accuracy do not match the state-of-the-art because the CNN is small and we do not use advanced tricks; we follow the settings of the seminal work \cite{mcmahan2017communication}.}
    \label{fig:CNN_loss_acc}
  \end{minipage}
    ~~~~~~~~
  \begin{minipage}[t]{.31\linewidth}
    \centering{
    \vspace{2mm}
    {\includegraphics[width=0.96\textwidth]{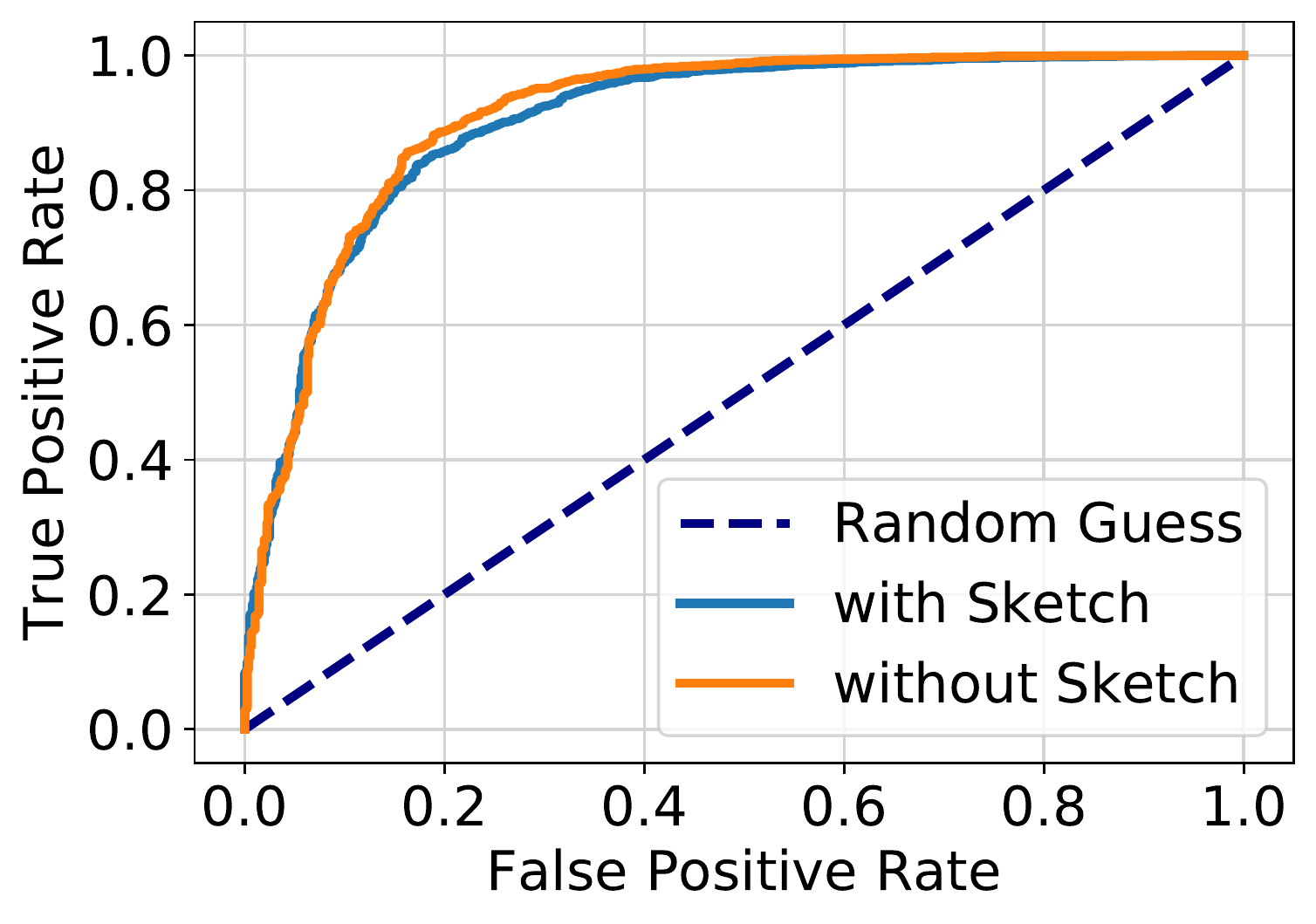}}}
    \vspace{2mm}
	\caption{Gender classification on the LWF dataset. }
    \label{fig:gender}
  \end{minipage}
  \hfill\mbox{}
\end{figure*}

\subsection{Accuracy and Efficiency} \label{sec:exp:acc}

We conduct experiments on the MNIST and CIFAR-10 datasets to show that first, DBCL does not hinder prediction accuracy, and second, it does not much increase the communication cost.
We follow the setting of \cite{mcmahan2017communication}.
The learning rates are tuned to optimize the convergence rate.

\paragraph{MNIST classification.}
We build a multilayer perceptron (MLP) and a convolutional neural network (CNN) for the multi-class classification task.
The MLP has 3 dense layers:
$\textsf{Dense} (200)$
$\Rightarrow$
\textsf{ReLU}
$\Rightarrow$
$\textsf{Dense} (200)$
$\Rightarrow$
\textsf{ReLU}
$\Rightarrow$
$\textsf{Dense} (10)$
$\Rightarrow$
\textsf{Softmax}.
The CNN has 2 convolutional layers and 2 dense layers:
$\textsf{Conv} (32, 5\times 5)$
$\Rightarrow$
\textsf{ReLU}
$\Rightarrow$
$\textsf{MaxPool} (2\times 2)$
$\Rightarrow$
$\textsf{Conv} (64, 5\times 5)$
$\Rightarrow$
\textsf{ReLU}
$\Rightarrow$
$\textsf{MaxPool} (2\times 2)$
$\Rightarrow$
\textsf{Flatten}
$\Rightarrow$
$\textsf{Dense} (512)$
$\Rightarrow$
\textsf{ReLU}
$\Rightarrow$
$\textsf{Dense} (10)$
$\Rightarrow$
\textsf{Softmax}.

We use Federated Averaging (\texttt{FedAvg}) to train the MLP and CNN.
The data are partitioned among $100$ (virtual) clients uniformly at random.
Between two communications, \texttt{FedAvg} performs local computation for 1 epoch (for MLP) or 5 epochs (for CNN).
The batch size of local SGD is set to $10$.

Sketching is applied to all the dense and convolutional layers except the output layer.
We set the sketch size to $s = \din / 2$; thus, the per-communication word complexity is reduced by half.
With sketching, the MLP and CNN are trained by \texttt{FedAvg} under the same setting.

We show the experimental results in Table~\ref{tab:ACC_sketch}.
Trained by \texttt{FedAvg}, the small MLP can only reach 97\% test accuracy, while the CNN can obtain 99\% test accuracy.
In this set of experiments, sketching does not hinder test accuracy at all.
We show the rounds of communications for attaining the test accuracies.
For the MLP, sketching needs 2.6x $\sim$ 3.6x rounds of communications to converge.
For the CNN, sketching needs 0.6x $\sim$ 5.6x rounds of communications to converge.
Using sketching, the per-communication word complexity is reduced to 0.5x.

\paragraph{CIFAR-10 classification.}
We build a CNN with 3 convolutional layers and 2 dense layers:
$\textsf{Conv} (32, 5\times 5)$
$\Rightarrow$
\textsf{ReLU}
$\Rightarrow$
$\textsf{Conv} (64, 5\times 5)$
$\Rightarrow$
\textsf{ReLU}
$\Rightarrow$
$\textsf{MaxPool} (2\times 2)$
$\Rightarrow$
$\textsf{Conv} (128, 5\times 5)$
$\Rightarrow$
\textsf{ReLU}
$\Rightarrow$
$\textsf{MaxPool} (2\times 2)$
$\Rightarrow$
\textsf{Flatten}
$\Rightarrow$
$\textsf{Dense} (200)$
$\Rightarrow$
\textsf{ReLU}
$\Rightarrow$
$\textsf{Dense} (10)$
$\Rightarrow$
\textsf{Softmax}.


The CNN is also trained using \texttt{FedAvg}.
The data are partitioned among $100$ clients.
We set the participation ratio to $c=10\% $, that is, each time only $10\% $ uniformly sampled clients participate in the training.
Between two communications, \texttt{FedAvg} performs local computation for 5 epochs.
The batch size of local SGD is set to $50$.
We follow the settings of \cite{mcmahan2017communication}, so we do not use tricks such as data augmentation, batch normalization, skip connection, etc.

Figure~\ref{fig:CNN_loss_acc} shows the convergence curves.
Using sketching does not hinder the test accuracy at all; on the contrary, it marginally improves the test accuracy.
The reason is likely that sketching is an adaptive regularization similar to dropout; see the discussions in Section~\ref{sec:theory:optimization}.


\paragraph{Binary classification on imbalanced data.}
Following \cite{melis2019exploiting}, we conduct binary classification experiments on a subset of the LFW dataset.
We use {$8,150$} faces for training and $3,400$ for test.
The task is gender prediction.
We build a CNN with 3 convolutional layers and 3 dense layers:
$\textsf{Conv} (64, 3\times 3)$
$\Rightarrow$
\textsf{ReLU}
$\Rightarrow$
$\textsf{MaxPool} (2\times 2)$
$\Rightarrow$
$\textsf{Conv} (64, 3\times 3)$
$\Rightarrow$
\textsf{ReLU}
$\Rightarrow$
$\textsf{MaxPool} (2\times 2)$
$\Rightarrow$
$\textsf{Conv} (128, 3\times 3)$
$\Rightarrow$
\textsf{ReLU}
$\Rightarrow$
$\textsf{MaxPool} (2\times 2)$
$\Rightarrow$
\textsf{Flatten}
$\Rightarrow$
$\textsf{Dense} (32)$
$\Rightarrow$
\textsf{ReLU}
$\Rightarrow$
$\textsf{Dense} (32)$
$\Rightarrow$
\textsf{ReLU}
$\Rightarrow$
$\textsf{Dense} (1)$
$\Rightarrow$
\textsf{Sigmoid}.
We apply sketching to all the convolutional and dense layers except the output layer.
The model is trained by distributed SGD (2 clients and 1 server) with a learning rate of $0.01$ and a batch size of $32$.


The dataset is class-imbalanced: {8957} are males, and {2593} are females.
We therefore use ROC curves, instead of classification accuracy, for evaluating the classification.
In Figure~\ref{fig:gender}, we plot the ROC curves to compare the standard CNN and the sketched one. 
The two ROC curves are almost the same.



\begin{figure*}[!t]
	\begin{center}
		\includegraphics[width=0.99\textwidth]{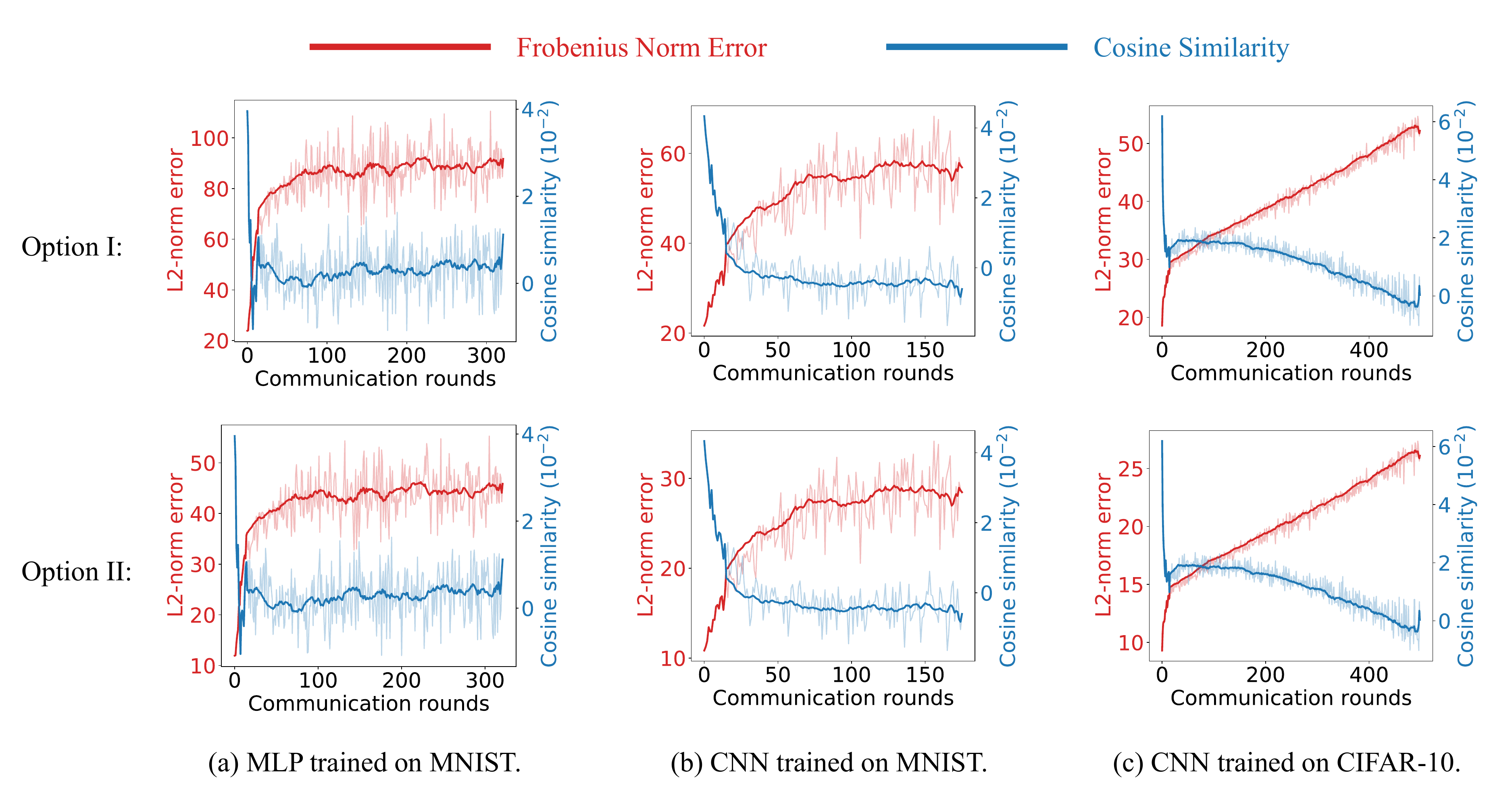}
		\vspace{-3mm}
		\caption{The x-axis is the communication rounds. The y-axes are Frobenius norm errors (red) and cosine similarities (blue).
		The figures show that the estimated gradient, $\widehat{\De}$, is far from the true gradient, $\De$, which means our defense works.
		}
		\label{fig:verify_bounds}
	\end{center}
\end{figure*}

\subsection{Defending Gradient-Based Attacks} \label{sec:exp:defense}

We empirically study whether DBCL can defend the gradient-based attacks of \cite{melis2019exploiting} and \cite{zhu2019deep}.
The details of experiment setting are in the appendix.

\paragraph{Gradient estimation.}
To perform gradient-based attacks, a client needs to (approximately) know the gradient (or other updating directions).
We discuss in Section~\ref{sec:theory:threat} how a client can estimate the updating directions.
We empirically study two approaches:
\begin{align*}
    \textrm{Option I:} \quad 
    & \widehat{\De} = \W_{\textrm{old}} \S_{\textrm{old}} \S_{\textrm{old}}^T - \W_{\textrm{new}} \S_{\textrm{new}} \S_{\textrm{new}}^T ,\\
    \textrm{Option II:} \quad 
    & \widehat{\De} = \W_{\textrm{old}} \S_{\textrm{old}} \S_{\textrm{old}}^\dag - \W_{\textrm{new}} \S_{\textrm{new}} \S_{\textrm{new}}^\dag .
\end{align*}
Here, $\A^\dag$ means the Moore-Penrose inverse of matrix $\A$.
Let $\De = \W_{\textrm{old}}  - \W_{\textrm{new}} $ be the true updating direction.
We evaluate the approximation quality using two metrics:
\begin{align*}
    \textrm{The $\ell_2$-norm error:} \quad 
    & \big\| \vect \big( \widehat{\De}  - \De \big) \big\|_2 \, \big/ \, \big\| \vect \big( \De \big) \big\|_2 , \\
    \textrm{Cosine similarity:} \quad 
    & \big\langle \, \vect \big( \widehat{\De}  \big) \, , \,  \vect \big( \De \big) \, \big\rangle .
\end{align*}
If $\widehat{\De} $ is far from $\De$, which means our defense is effective, then the $\ell_2$ error is big, and the cosine similarity is small.

Figure~\ref{fig:verify_bounds} plots the quality of gradient estimation.
The experiment settings are the same as Section~\ref{sec:exp:acc}; the participation ratio is set to $10\%$.
The results show that $\widehat{\De}$ is entirely different from $\De$, which means our defense works.
Equation \eqref{eq:implicit_noise} implies that as the magnitude of $\De$ decreases, $\widehat{\De}$ is dominated by noise, and thus the error $\frac{ \| \vect ( \widehat{\De}  - \De ) \|_2 }{ \| \vect (\De ) \|_2  }$ gets bigger.
In our experiments, as the communication rounds increase, the algorithm tends to converge, the magnitude of $\De$ decreases, and the error increases.
The empirical observations verify our theories.

\paragraph{Defending the property inference attack (PIA) of \citet{melis2019exploiting}.}
We conduct experiments on the LFW dataset.
The task of collaborative learning is the gender classification under the same setting as Section~\ref{sec:exp:acc}.
The attacker seeks to infer whether a single batch of photos in the victim's private dataset contain Africans or not.
We use one server and two clients, which is the easiest for the attacker.
\begin{itemize}
    \item 
    Let one client be the attacker and the other client be the victim.
    Without sketching, the AUC is 1.0, which means the attacker is always right.
    With sketching applied, the AUC is 0.5, which means the performance is the same as random guess.
\vspace{-1mm}
    \item
    Let the server be the attacker and one client be the victim.
    Without sketching, the AUC is 1.0, which means the attacker succeeds.
    With sketching applied, the AUC is $0.726$.
    Our defense makes server-side attack much less effective, but users' privacy can still leak.
\end{itemize}

\paragraph{Defending the gradient-matching attack of \citet{zhu2019deep}.}
The attack seeks to recover the victims' data using model parameters and gradients.
They seek to find a batch of images by optimization so that the resulting gradient matches the observed gradient of the victim.
We use the same CNNs as \cite{zhu2019deep} to conduct experiments on the MNIST and CIFAR-10 datasets.
Without using sketching, the gradient-matching attack very well recovers the images.
For both client-side attacks and server-side attacks, if sketching is applied to all except the output layer, then the recovered images are just like random noise.
The experiments show that we can defend gradient-matching attack performed by both server and client.

\section{Related Work} \label{sec:related}

Cryptography approaches such as secure aggregation \cite{bonawitz2017practical}, homomorphic encryption \cite{aono2017privacy,giacomelli2018privacy,gilad2016cryptonets,liu2018secure,zhang2018gelu}, 
Yao’s garbled circuit protocol~\cite{rouhani2018deepsecure}, and many other methods \cite{yuan2013privacy,zhang2018survey} can also improve the security of collaborative learning.
Generative models such as \cite{chen2018differentially} can also improve privacy; however, they hinder the accuracy and efficiency, and their tuning and deployment are nontrivial.
All the mentioned defenses are not competitive methods of our DBCL; instead, they can be combined with DBCL to defend more attacks.

Our methodology is based on matrix sketching \cite{johnson1984extensions,drineas2008cur,halko2011ramdom,mahoney2011ramdomized,woodruff2014sketching,drineas2016randnla}.
Sketching has been applied to achieve differential privacy \cite{blocki2012johnson,kenthapadi2012privacy,li2019privacy}.
It has been shown that to protect privacy, matrix sketching has the same effect as injecting random noise.
Our method is developed based on the connection between sketching \cite{woodruff2014sketching} and dropout training \cite{srivastava2014dropout};
in particular, if $\S$ is uniform sampling, then DBCL is essentially dropout.
Our approach is similar to the contemporaneous work \cite{khaled2019gradient} which is developed for computational benefits.

Decentralized learning, that is, the clients perform peer-to-peer communication without a central server, is an alternative to federated learning and has received much attention in recent years \cite{colin2016gossip,koloskova2019decentralized,lan2017communication,luo2019heterogeneity,ram2010asynchronous,sirb2016consensus,tang2018d,wang2019matcha,yuan2016convergence}.
The attacks of \citet{melis2019exploiting,zhu2019deep} can be applied to decentralized learning, and DBCL can defend the attacks under the decentralized setting.
We discuss decentralized learning in the appendix.
The attacks and defense under the decentralized setting will be our future work.

\section{Conclusions}

Collaborative learning enables multiple parties to jointly train a model without data sharing.
Unfortunately, standard distributed optimization algorithms can easily leak participants' privacy.
We proposed Double-Blind Collaborative Learning (DBCL) for defending gradient-based attacks which are the most effective privacy inference methods.
We showed that DBCL can defeat gradient-based attacks conducted by malicious clients.
Admittedly, DBCL can not defend all kinds of attacks; for example, if the server is malicious, then the attack of \cite{melis2019exploiting} still works, but much less effectively.
While it improves privacy, DBCL does not hurt test accuracy at all and does not much increase the cost of training.
DBCL is easy to use and does not need extra tuning.
Our future work will combine DBCL with cryptographic methods such as homomorphic encryption and secret sharing so that neither client nor server can infer users' privacy.

\section*{Acknowledgements}

The author thanks Michael Mahoney, Richard Peng, Peter Richt{\'a}rik, and David Woodruff for their helpful suggestions.

%
%
%

\bibliography{bib/ml,bib/distributed,bib/random,bib/decentralized,bib/optimization}
\bibliographystyle{icml2021}

\newpage

\onecolumn

\appendix

\input{Appendix}

\end{document}

%% file: Appendix.tex
\section{Details of Experimental Setting} \label{sec:exp2}

\subsection{Server-Side Attacks}

Throughout the paper, we focus on client-side attacks.
We empirically show that DBCL can perfectly defend client side attacks.
Here we briefly discuss server-side attacks.
We assume first, the server is honest but curious, second, the server holds a subset of data, and third, the server knows the true model parameters $\W$, the sketched gradients, and the sketch matrix, $\S$.
The assumptions, especially the third, make it easy to attack but hard to defend.
Let $\Gam_i$ be the sketched gradient of the $i$-th client.
The server uses $\Gam_i \S^T$ to approximate the true gradient; see Eqn \eqref{eq:grad_sketch}.
The server uses $\Gam_i \S^T$ for privacy inference.
\vspace{-2mm}

\subsection{Defending the Property Inference Attack (PIA) of \citet{melis2019exploiting}}

We conduct experiments on the LFW dataset by following the settings of \cite{melis2019exploiting}.
We use one server and two clients.
The task of collaborative learning is the gender classification discussed in Section~\ref{sec:exp:acc}.
The two clients collaboratively train the model on the gender classification task for $20,000$ iterations.

During the training, the attacker seeks to infer whether a single batch of photos in the victim's private dataset contain Africans or not.
We consider two types of attackers.
First, one client is the attacker, and one client is the victim.
Section, the server is the attacker, and one client is the victim.
In the latter case, the server locally holds a subset of data; the sample size is the same as the clients.

During the $1,000$th to the $20,000$th iterations of the training, the attacker estimate the gradients computed on the victim's private data.
Taking the estimated gradients are input features, a random forest performs property inference attack, that is, infer whether a batch contains Africans or not.
In each iteration, the attacker trains the random forest using its local data.
More specifically, the attacker computes $2$ gradients with property (i.e., contain Africans) and $8$ gradients without property, and the attacker takes the gradients as input features for training the random forest
The attacker uses a total of $190,000$ gradients for training the random forest and $19,000$ gradients for testing.

After training the random forest, the attacker uses the random forest for binary classification.
The task is to infer whether a batch of the victim's (a client) private images contain Africans or not.
The test data are class-imbalanced: only $3,800$ gradients are from images of Africans, whereas the rest $15,200$ are from non-Africans.
We thus use AUC as the evaluation metric.
Without sketching, the AUC is $1.0$, which means the attacker can exactly tell whether a batch of client's images contain Africans or not.
Using sketching, the AUC of client-side attack is $0.5$, and the AUC of server-side attack is $0.726$.

\subsection{Defending the Gradient-Matching Attack of \citet{zhu2019deep}}

Zhu \etal \cite{zhu2019deep} proposed to recover the victims' data using model parameters and gradients.
They seek to find a batch of images by optimization so that the resulting gradient matches the observed gradient of the victim.
We use the same CNNs as \cite{zhu2019deep} to conduct experiments on the MNIST and CIFAR-10 datasets.
We apply sketching to all except the output layer.

We show in Figure~\ref{fig:gradient_leakage} the results of server-side attacks.
Without sketching, the gradient-matching attack very well recovers the private data of the victim.
However, with sketching, the gradient-matching attack completely fails.
Moreover, the experiments on client-side attacks have the same results.

\begin{figure}[h]
	\centering
	\includegraphics[width=0.8\textwidth]{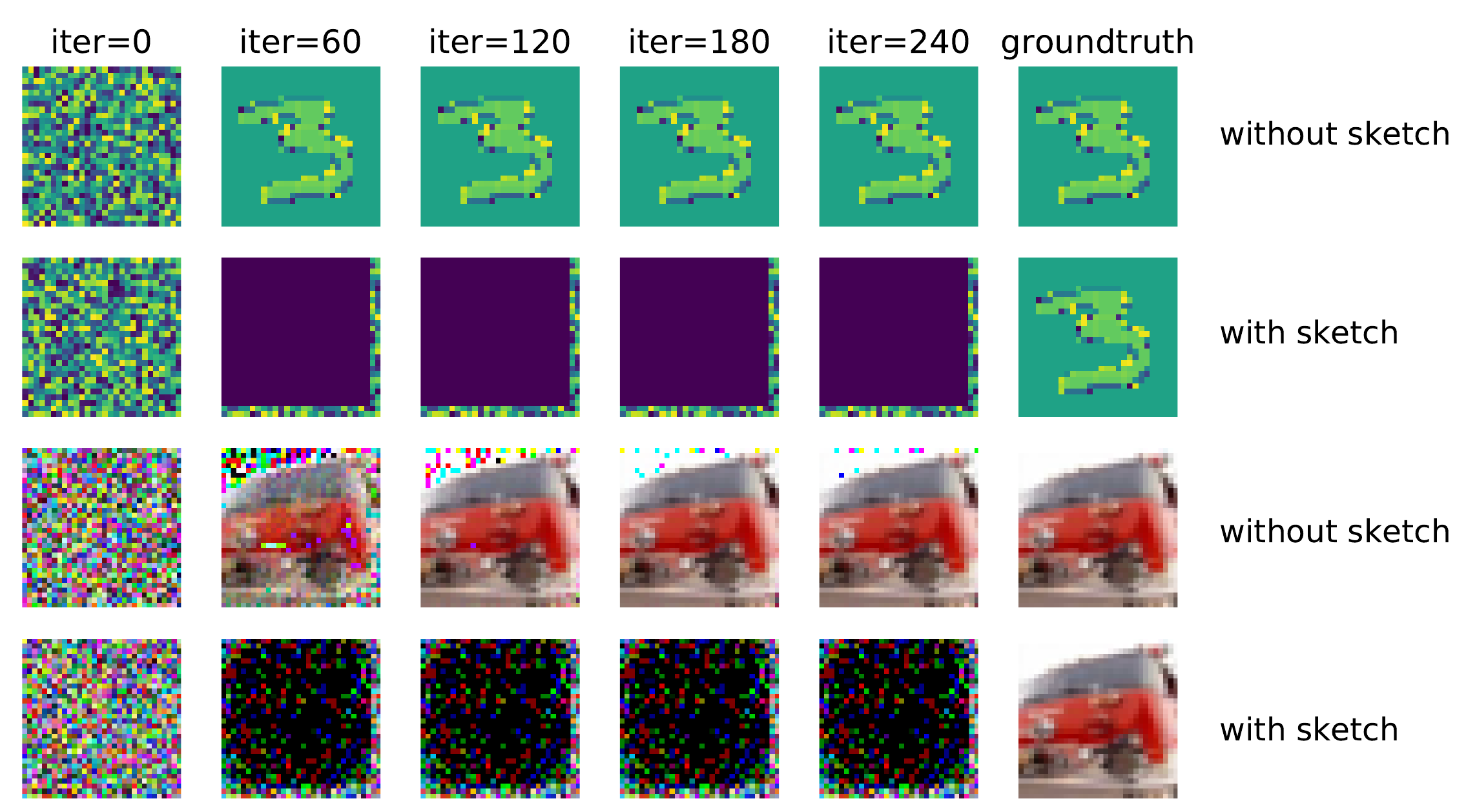}
	\caption{The images are generated by the gradient-matching attack of \cite{zhu2019deep}. 
	The attack is effective for the standard CNNs.
	Using sketching, the gradient-matching attack cannot recover the images.}
	\label{fig:gradient_leakage}
\end{figure}

%
%
%
%

\section{Algorithm Derivation} \label{sec:derive}

In this section, we derive the algorithm outlined in Section~\ref{sec:alg}.
In Section~\ref{sec:derive:dense} and \ref{sec:conv}, we apply sketching to dense layer and convolutional layer, respectively, and derive the gradients.

\subsection{Dense Layers}\label{sec:derive:dense}

For simplicity, we study the case of batch size $b=1$ for a dense layer.
Let $\x \in \RB^{1\times \din} $ be the input, $\W \in \RB^{\dout\times \din} $ be the parameter matrix, 
$\S \in \RB^{\din \times s}$ ($s < \din $) be a sketching matrix, and 
\begin{equation*}
\z \: = \: \x \S (\W \S )^T\in \RB^{1 \times \dout } 
\end{equation*}
be the output (during training).
For out-of-sample prediction, sketching is not applied, equivalently, $\S = \I_{\din}$.

In the following, we describe how to propagate gradient from the loss function back to $\x$ and $\W$.
The dependence among the variables can be depicted as
\begin{small}
	\begin{equation*}
	\begin{array}{c}
	\textsf{input} \: \longrightarrow \:  \cdots \: \longrightarrow   \\
	~
	\end{array}
	\underbrace{
		\left.
		\begin{array}{c}
		\x   \\
		\W \\
		\end{array}
		\right\}
		\: \longrightarrow \: 
		\z }_{\textsf{the studied layer}} 
	\: \longrightarrow \: 
	\cdots
	\: \longrightarrow \: 
	\textsf{loss}.
	\end{equation*}
\end{small}%

During the backpropagation, the gradients propagated to the studied layer are
\begin{equation} \label{eq:bp_grad_out}
\g \: \triangleq \: \frac{ \partial \, L }{ \partial \, \z } 
\: \in \: \RB^{1\times \dout } ,
\end{equation}
where $L$ is some loss function.
Then we further propagate the gradient from $\z$ to $\x$ and $\W$:
\begin{align}
& \frac{ \partial \, L }{ \partial \, \x } 
\: = \: \frac{ \partial \, L }{ \partial \, \z }  \,
\frac{ \partial \, \z }{ \partial \, (\x \S ) } \,
\frac{ \partial \, (\x \S ) }{ \partial \, \x } 
\: = \: \g (\W \S ) \S^T
\: \in \: \RB^{1 \times \din  } ,  \label{eq:grad_L_X_sketch}  \\
& \frac{ \partial \, L }{ \partial \, \W } 
\: = \: \g^T (\x \S) \S^T
\: \in \: \RB^{\dout \times \din } . \label{eq:grad_L_W_sketch} 
\end{align}
We prove \eqref{eq:grad_L_W_sketch}  in the following.
Let $\tilde{\x} = \x \S \in \RB^{1\times s}$ and $\widetilde{\W} = \W \S \in \RB^{\dout \times s}$.
Let $\w_{j:}$ and $\tilde{\w}_{j:}$ be the $j$-th row of $\W$ and $\widetilde{\W}$, respectively.
It can be shown that 
\begin{equation*}
\frac{ \partial \, L }{ \partial \, \tilde{\w}_{j:} } 
\: = \: \sum_{l=1}^{\dout} \frac{ \partial \, L }{ \partial \, z_l } \,
\frac{ \partial \, z_l }{ \partial \, \tilde{\w}_{j:} } 
\: = \: \sum_{l=1}^{\dout}  \frac{ \partial \, L }{ \partial \, z_l } \,
\frac{ \partial \,( \tilde{\x} \tilde{\w}_{l:}^T) }{ \partial \, \tilde{\w}_{j:} } 
\: = \: \frac{ \partial \, L }{ \partial \, z_j } \,
\frac{ \partial \,( \tilde{\x} \tilde{\w}_{j:}^T) }{ \partial \, \tilde{\w}_{j:} } 
\: = \: g_j \, \tilde{\x}
\: \in \: \RB^{1\times s} .
\end{equation*}
Thus, $\tilde{\w}_{j:} = \w_{j:} \S \in \RB^{1\times s}$;
moreover, $\tilde{\w}_{j:}$ is independent of $\w_{l:}$ if $j\neq l$.
It follows that
\begin{equation*}
\frac{ \partial \, L }{ \partial \, {\w}_{j:} } 
\: = \: \frac{ \partial \, L }{ \partial \, \tilde{\w}_{j:} } \,
\frac{ \partial \, \tilde{\w}_{j:} }{ \partial \, {\w}_{j:} } 
\: = \: g_j \, \tilde{\x} \, \S^T
\: \in \: \RB^{1\times \din} .
\end{equation*}
Thus
\begin{equation*}
\frac{ \partial \, L }{ \partial \, \W }
\: = \:\g^T \, \tilde{\x} \, \S^T
\: = \: \g^T (\x \S) \S^T
\: \in \: \RB^{\dout \times \din} .
\end{equation*}

\subsection{Extension to Convolutional Layers} \label{sec:conv}

Let $\X$ be a $d_1 \times d_2 \times d_3$ tensor and $\K$ be a $k_1\times k_2 \times d_3$ kernel.
The convolution $\X * \K$ outputs a $d_1   \times d_2  $ matrix (assume zero-padding is used).
The convolution can be equivalently written as matrix-vector multiplication in the following way.

We segment $\X$ to many patches of shape $k_1\times k_2\times d_3$  and then reshape every patch to a $\din \triangleq k_1 k_2 d_3$-dimensional vector.
Let $\pp_i$ be the $i$-th patch (vector).
Tensor $\X$ has $q \triangleq d_1  d_2$ such patches.
Let
\begin{equation*}
\PP \triangleq [\pp_1 , \cdots , \pp_q]^T \in \RB^{q \times \din}
\end{equation*}
be the concatenation of the patches. 
Let $\w \in \RB^{\din}$ be the vectorization of the kernel $\K \in \RB^{k_1\times k_2 \times d_3}$.
The matrix-vector product, $\z = \PP \w \in \RB^q$, is indeed the vectorization of the convolution $\X * \K$.

In practice, we typically use multiple kernels for the convolution; 
let $\W \triangleq [\w_1 , \cdots , \w_{\dout}]^T \in \RB^{\dout \times \din}$ be the concatenation of $\dout$ different (vectorized) kernels.
In this way, the convolution of $\X$ with $\dout$ different kernels, which outputs a $d_1  \times d_2  \times r$ tensor, is the reshape of $\X \W^T \in \RB^{q\times \dout}$.

We show in the above that tensor convolution can be equivalently expressed as matrix-matrix multiplication.
Therefore, we can apply matrix sketching to convolutional layers in the same way as the dense layer.
Specifically, let $\S$ be a $\din \times s$ random sketching matrix.
Then $\X \S (\W \S)^T$ is an approximation to $\X \W^T$, and the backpropagation is accordingly derived using matrix differentiation.

\section{Proofs}

In this section, we prove Theorem~\ref{thm:privacy} and Corollary~\ref{cor:privacy}.
Theorem~\ref{thm:privacy} follows from Lemmas~\ref{lemma:countsketch0} and \ref{lem:countsketch1}.
Corollary~\ref{cor:privacy} follows from Lemma~\ref{lem:countsketch2}.
Theorem~\ref{thm:privacy2} is a trivial consequence of Theorem~\ref{thm:privacy}.

\begin{lemma} \label{lemma:countsketch0}
	Let $\S_{\textrm{old}} $ and $\S_{\textrm{new}} $ be independent CountSketch matrices.
	For any matrix $\V$ independent of $\S_{\textrm{old}} $ and $\S_{\textrm{new}} $, the following identity holds:
	\begin{align*}
	& \EB \big\| \widehat{\De}  \V^T - \De \V^T \big\|_F^2 \\
	& = \: \EB \big\| \W_{\textrm{old}}  \V^T -  \W_{\textrm{old}} \S_{\textrm{old}} \S_{\textrm{old}}^T \V^T \big\|_F^2
	+ \EB \big\| \W_{\textrm{new}}  \V^T -  \W_{\textrm{new}} \S_{\textrm{new}} \S_{\textrm{new}}^T \V^T \big\|_F^2 ,
	\end{align*}
	where the expectation is taken w.r.t.\ the random sketching matrices $\S_{\textrm{old}}$ and $\S_{\textrm{new}}$.
\end{lemma}

\begin{proof}
	Recall the definitions: $\De = \W_{\textrm{old}} - \W_{\textrm{new}}$
	and $\widehat{\De} = \W_{\textrm{old}} \S_{\textrm{old}} \S_{\textrm{old}}^T - \W_{\textrm{new}} \S_{\textrm{new}} \S_{\textrm{new}}^T$.
	Then
	\begin{align*}
	& \big\| \widehat{\De}  \V^T - \De \V^T \big\|_F^2
	\: = \: \Big\| \big( \W_{\textrm{old}} \S_{\textrm{old}} \S_{\textrm{old}}^T -  \W_{\textrm{new}} \S_{\textrm{new}} \S_{\textrm{new}}^T \big)  \V^T - \De \V^T \Big\|_F^2 \\
	& = \: \Big\| \big( \W_{\textrm{old}} \S_{\textrm{old}} \S_{\textrm{old}}^T -  \W_{\textrm{new}} \S_{\textrm{new}} \S_{\textrm{new}}^T \big)  \V^T - (\W_{\textrm{old}} - \W_{\textrm{new}})  \V^T \Big\|_F^2 \\
	& = \: \Big\|\W_{\textrm{old}}   \big( \S_{\textrm{old}} \S_{\textrm{old}}^T -  \I \big)  \V^T 
	+ \W_{\textrm{new}}   \big( \I - \S_{\textrm{new}} \S_{\textrm{new}}^T  \big) \V^T\Big\|_F^2 \\
	& = \: \big\| \W_{\textrm{old}}   \big( \S_{\textrm{old}} \S_{\textrm{old}}^T -  \I \big)  \V^T  \big\|_F^2
	+ \big\| \W_{\textrm{new}}   \big( \I - \S_{\textrm{new}} \S_{\textrm{new}}^T  \big) \V^T  \big\|_F^2 \\
	& \quad + 2 \Big\langle \W_{\textrm{old}}   \big( \S_{\textrm{old}} \S_{\textrm{old}}^T -  \I \big)  \V^T  , \;
	\W_{\textrm{new}}   \big( \I - \S_{\textrm{new}} \S_{\textrm{new}}^T  \big) \V^T \Big\rangle.
	\end{align*}
	Since $\EB [ \S_{\textrm{old}} \S_{\textrm{old}}^T - \I  ] = \0$, $\EB [\S_{\textrm{new}} \S_{\textrm{new}}^T - \I] = \0$, and $\S_{\textrm{old}}$ and $\S_{\textrm{new}}$ are independent, we have
	\begin{align*}
	\EB \Big[ \Big\langle \W_{\textrm{old}}   \big( \S_{\textrm{old}} \S_{\textrm{old}}^T -  \I \big)  \V^T  , \;
	\W_{\textrm{new}}   \big( \I - \S_{\textrm{new}} \S_{\textrm{new}}^T  \big) \V^T \Big\rangle \Big]
	\: = \:  0.
	\end{align*}
	It follows that
	\begin{align*}
	& \EB \big\| \widehat{\De}  \V^T - \De \V^T \big\|_F^2 \\
	& = \: \EB \big\|\W_{\textrm{old}}   \big( \S_{\textrm{old}} \S_{\textrm{old}}^T -  \I \big)  \V^T \big\|_F^2
	+ \EB \big\| \W_{\textrm{new}}   \big( \I - \S_{\textrm{new}} \S_{\textrm{new}}^T  \big) \V^T \big\|_F^2 ,
	\end{align*}
	by which the lemma follows.
\end{proof}

\begin{lemma} \label{lem:countsketch1}
	Let $\S $ be a ${d\times s}$ CountSketch matrix. 
	Let $\A \in \RB^{n\times d}$ and $\B \in \RB^{m\times d}$ be any non-random matrices.
	Then
	\begin{align*}
	\EB \big[ \A \S \S^T \B^T - \A \B^T \big]^2
	\: = \: \frac{1}{s} \sum_{i=1}^n \sum_{j=1}^m \bigg( \sum_{k\neq l} a_{ki}^2 b_{lj}^2 +  \sum_{k\neq l} a_{ki} b_{kj} a_{li} b_{lj} \bigg) .
	\end{align*}
\end{lemma}

\begin{proof}
	\cite{pham2013fast,weinberger2009feature}
	showed that for any vectors $\a , \bb \in \RB^{d}$,
	\begin{align*} 
	& \EB \big[ \a^T \S \S^T \bb \big]
	\: = \: \a^T \bb, \nonumber \\
	& \EB \big[ \a^T \S \S^T \bb - \a^T \bb \big]^2
	\: = \: \frac{1}{s} \Big( \sum_{k\neq l} a_k^2 b_l^2 +  \sum_{k\neq l} a_k b_k a_l b_l \Big) .
	\nonumber 
	\end{align*}
	Let $\a_{i:} \in \RB^d$ be the $i$-th row of $\A \in \RB^{n\times d}$ and $\bb_{j:} \in \RB^d$ be the $j$-th column of $\B \in \RB^{m\times d}$.
	Then,
	\begin{align*}
	& \EB \big[ \A \S \S^T \B^T - \A \B^T \big]^2
	\: = \: \sum_{i=1}^n \sum_{j=1}^m \EB \big[ \a_{i:}^T \S \S^T \bb_{j:} - \a_{i:}^T \bb_{j:} \big]^2 \\
	& = \: \frac{1}{s} \sum_{i=1}^n \sum_{j=1}^m \bigg( \sum_{k\neq l} a_{ik}^2 b_{jl}^2 +  \sum_{k\neq l} a_{ik} b_{jk} a_{il} b_{jl} \bigg) ,
	\end{align*}
	by which the lemma follows.
\end{proof}

\begin{lemma}\label{lem:countsketch2}
	Let $\S $ be a ${d\times s}$ CountSketch matrix. 
	Assume that the entries of $\A$ are IID and that the entries of $\B$ are also IID.
	Then
	\begin{align*}
	\EB \big[ \A^T \S \S^T \B - \A^T \B \big]^2
	\: = \: \Theta \Big( \frac{d }{s}  \Big) \| \A \B^T \|_F^2  .
	\end{align*}
\end{lemma}

\begin{proof}
	Assume all the entries of $\A$ are IID sampled from a distribution with mean $\mu_A$ and standard deviation $\sigma_A$;
	assume all the entries of $\B$ are IID sampled from a distribution with mean $\mu_B$ and standard deviation $\sigma_B$.
	It follows from Lemma~\ref{lem:countsketch1} that 
	\begin{align*}
	& \EB \big[ \A \S \S^T \B^T - \A \B^T \big]^2
	\: = \: \frac{m n }{s} \big[  (d^2-d) (\mu_A^2 + \sigma_A^2) (\mu_B^2 + \sigma_B^2) +  (d^2-d) \mu_A^2 \mu_B^2 \big] \\
	& = \: \Theta \Big( \frac{m n d^2 }{s} (\mu_A^2 + \sigma_A^2) (\mu_B^2 + \sigma_B^2)  \Big)
	\: = \: \Theta \Big( \frac{d }{s}  \Big) \| \A \B^T \|_F^2 .
	\end{align*}
\end{proof}

\section{Decentralized Learning} \label{sec:decentralized}

\begin{wrapfigure}{r}{0.45\textwidth} 
	\vspace{-8pt}
	\begin{center}
		\includegraphics[width=0.4\textwidth]{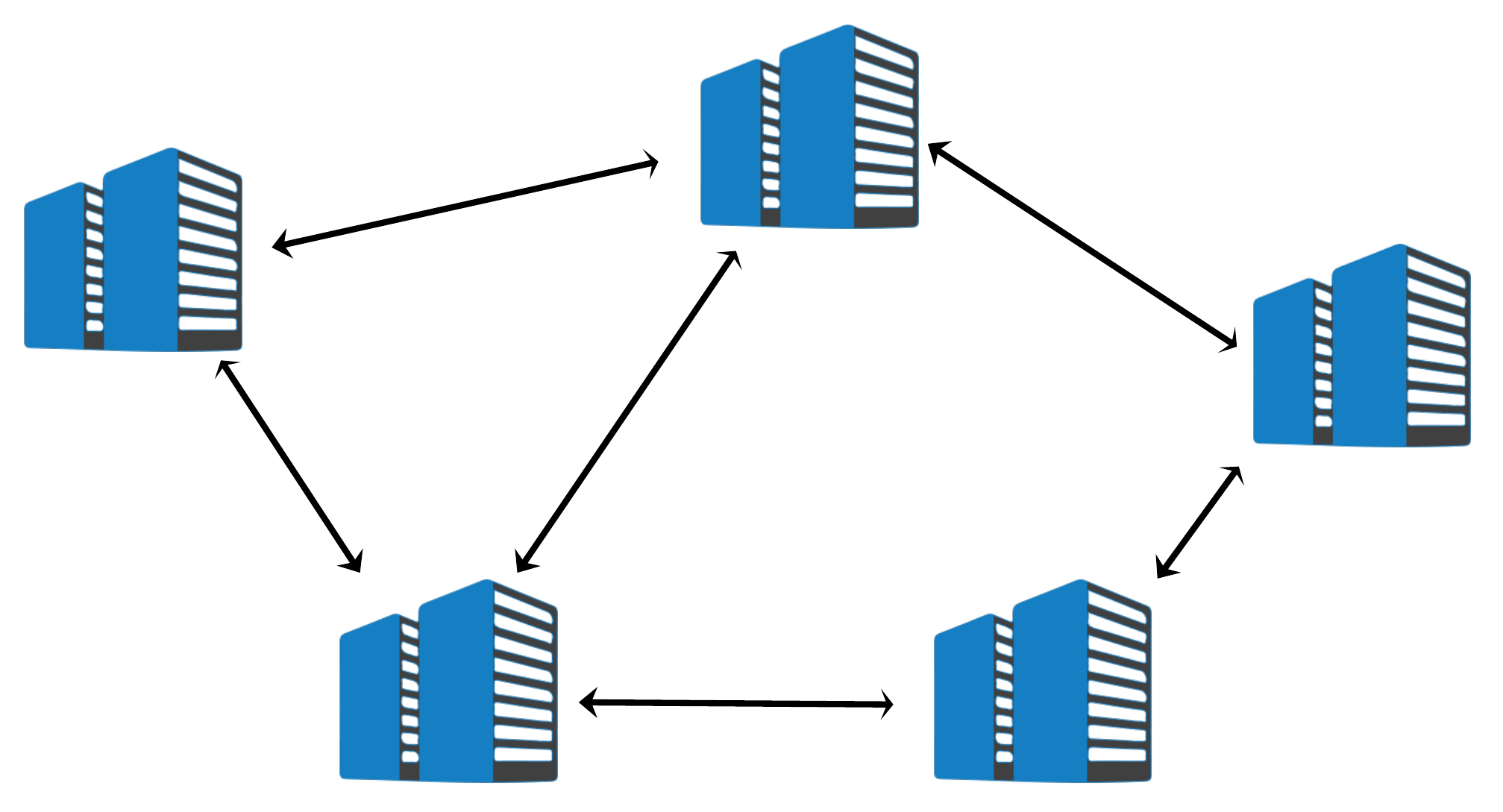}
		\caption{Decentralized learning in a peer-to-peer network.}
		\label{fig:decentralized}
	\end{center}
	\vspace{-20pt}
	\vspace{1pt}
\end{wrapfigure} 

Instead of relying on a central server, multiple parties can collaborate using such a peer-to-peer network as Figure~\ref{fig:decentralized}.
A client is a compute node in the graph and connected to a few neighboring nodes.
Many decentralized optimization algorithms have been developed~\cite{bianchi2013performance,yuan2016convergence,sirb2016consensus,colin2016gossip,lian2017can,lan2017communication,tang2018d}.
The nodes collaborate by, for example, aggregating its neighbors' model parameters, taking a weighted average of neighbors' and its own parameters as the intermediate parameters, and then locally performing an SGD update.


We find that the attacks of \cite{melis2019exploiting,zhu2019deep} can be applied to this kind of decentralized learning.
Note that a node shares its model parameters with its neighbors.
If a node is malicious, it can use its neighbors' gradients and model parameters to infer their data.
Let a neighbor's (the victim) parameters in two consecutive rounds be $\W_{\textrm{old}}$ and $\W_{\textrm{new}}$.
The difference, $\De = \W_{\textrm{old}} - \W_{\textrm{new}}$, is mainly the gradient evaluated on the victim's data.\footnote{Besides the victim's gradient, $\De$ contains the victim's neighbors' gradients, but their weights are usually lower than the victim's gradient.}
With the model parameters $\W$ and updating direction $\De$ at hand, the attacker can perform the gradient-based attacks of \cite{melis2019exploiting,zhu2019deep}.

Our DBCL can be easily applied under the decentralized setting:
two neighboring compute nodes agree upon the random seeds, sketching their model parameters, and communicate the sketches.
This can stop any node from knowing, $\De = \W_{\textrm{old}} - \W_{\textrm{new}}$, i.e., the gradient of the neighbor.
We will empirically study the decentralized setting in our future work.

\newpage

\lstset{ %
extendedchars=false,            
language=Python,                
xleftmargin=1em,
xrightmargin=1em,
basicstyle=\footnotesize,    
tabsize=3,                            
numbers=left,                   
numberstyle=\tiny,              
stepnumber=1,                   
numbersep=5pt,                  
keywordstyle=\color[rgb]{0,0,1},                
commentstyle=\color[rgb]{0.133,0.545,0.133},    
stringstyle=\color[rgb]{0.627,0.126,0.941},      
backgroundcolor=\color{white}, 
showspaces=false,               
showstringspaces=false,         
showtabs=false,                 
frame=single,                 
breaklines=true,                
breakatwhitespace=false,        
mathescape=true,escapechar=?    
escapeinside={\%*}{*)},         
}

\section{Code of Core Algorithms}

The following code implements the CountSketch and transposed CountSketch.
Given matrix $\A$, the CountSketch outputs $\C = \A \S$.
Given $\C$, the transposed CountSketch outputs $\B = \C \S^T$.

\begin{lstlisting}[language=Python]
class Sketch():

    # Random Hashing
    # Generate random indices and random signs
    # Args:
    #    n: (integer) number of items to be hashed
    #    q: (integer) map n items to a table of s=n/q rows
    # Return:
    #    hash_idx: (q-by-s Torch Tensor) contain random integer in {0, 1, ..., s-1}
    #    rand_sgn: (n-by-1 Torch Tensor) contain random signs (+1 or -1)
    def rand_hashing(n, q):
        s = math.floor(n / q)
        t = torch.randperm(n)
        hash_idx = t[0:(s * q)].reshape((q, s))
        rand_sgn = torch.randint(0, 2, (n,)).float() * 2 - 1
        return hash_idx, rand_sgn

    # Count sketch
    # It converts m-by-n matrix to m-by-s matrix
    # Args:
    #    a: (m-by-n Torch Tensor) input matrix
    #    hash_idx: (q-by-s Torch Tensor) contain random integer in {0, 1, ..., s-1}
    #    rand_sgn: (n-by-1 Torch Tensor) contain random signs (+1 or -1)
    # Return:
    #    c: m-by-s sketch (Torch Tensor) (result of count sketch)
    def countsketch(a, hash_idx, rand_sgn):
        m, n = a.shape
        s = hash_idx.shape[1]
        c = torch.zeros([m, s], dtype=torch.float32)
        b = a.mul(rand_sgn)

        for h in range(s):
            selected = hash_idx[:, h]
            c[:, h] = torch.sum(b[:, selected], dim=1)

        return c

    # Transpose count sketch
    # The "countsketch" function converts m-by-n matrix A to m-by-s matrix C
    # This function maps C back to a m-by-n matrix B
    # Args:
    #    c: (m-by-s Torch Tensor) input matrix
    #    hash_idx: (q-by-s Torch Tensor) contain random integer in {0, 1, ..., s-1}
    #    rand_sgn: (n-by-1 Torch Tensor) contain random signs (+1 or -1)
    # Return:
    #    b: m-by-n matrix
    def transpose_countsketch(c, hash_idx, rand_sgn):
        m, s = c.shape
        n = len(rand_sgn)
        b = torch.zeros([m, n], dtype=torch.float32)
        for h in range(s):
            selected = hash_idx[:, h]
            b[:, selected] = c[:, h].reshape(m, 1)
        b = b.mul(rand_sgn)
        return b
\end{lstlisting}

\newpage

The standard linear function of PyTorch computes $\Z = \X \W^T + \B$, where $\X$ is a batch of inputs, $\W$ is the weight matrix, and $\B$ is the bias (aka intercept).
With CountSketch applied, the output becomes $\Z = \X \S \S^T \W^T + \B$.
We need to implement both the forward function and the backward function.
The backward function is called during backpropagation.

\begin{lstlisting}[language=Python]
class SketchLinearFunction(torch.autograd.Function):
    @staticmethod
    def forward(ctx, input, weight, bias, hash_idx, rand_sgn, training=True, q=2):
        if training:
            # input_features = weight.shape[-1]

            # sketching the input and weight matrices
            # hash_idx, rand_sgn = Sketch.rand_hashing(input_features, q)
            input_sketch = Sketch.countsketch(input, hash_idx, rand_sgn)
            weight_sketch = Sketch.countsketch(weight, hash_idx, rand_sgn)
            output = input_sketch.mm(weight_sketch.t())

            ctx.save_for_backward(input_sketch, weight_sketch, bias, hash_idx, rand_sgn)
        else:
            output = input.mm(weight.t())

        output += bias.unsqueeze(0).expand_as(output)
        return output

    @staticmethod
    def backward(ctx, grad_output):
        input_sketch, weight_sketch, bias, hash_idx, rand_sgn = ctx.saved_tensors
        grad_input = grad_weight = grad_bias = grad_training = None

        if ctx.needs_input_grad[0]:
            grad_input_tmp = grad_output.mm(weight_sketch)
            grad_input = Sketch.transpose_countsketch(grad_input_tmp, hash_idx, rand_sgn)
        if ctx.needs_input_grad[1]:
            grad_weight_tmp = grad_output.t().mm(input_sketch)
            grad_weight = Sketch.transpose_countsketch(grad_weight_tmp, hash_idx, rand_sgn)
        if ctx.needs_input_grad[2]:
            grad_bias = grad_output.sum(0).squeeze(0)

        return grad_input, grad_weight, grad_bias, None, None, None, None
\end{lstlisting}

\newpage

The following PyTorch code applies CountSketch to the parameter matrices of a dense layer (aka linear layer or fully-connected layer).

\begin{lstlisting}[language=Python]
class SketchLinear(nn.Module):
    def __init__(self, input_features, output_features, q=2):
        super(SketchLinear, self).__init__()
        self.input_features = input_features
        self.output_features = output_features
        self.q = q

        self.weight = nn.Parameter(torch.Tensor(output_features, input_features))
        self.bias = nn.Parameter(torch.Tensor(output_features))
        self.register_parameter('weight', self.weight)
        self.register_parameter('bias', self.bias)

        bound = 1 / math.sqrt(input_features)
        scaling = math.sqrt(3.0)
        self.weight.data.uniform_(-bound * scaling, bound * scaling)
        self.bias.data.uniform_(-bound, bound)

    def forward(self, input, hash_idx, rand_sgn):
        return SketchLinearFunction.apply(input, self.weight, self.bias, hash_idx, rand_sgn, self.training, self.q)

    def extra_repr(self):
        return 'input_features={}, output_features={}, weight={}, bias={}'.format(
            self.input_features, self.output_features, self.weight, self.bias
        )
\end{lstlisting}

\newpage

The following code applies CountSketch to the convolution function of PyTorch.
We express convolution as matrix multiplication, and then apply CountSketch in the same way as above.

\begin{lstlisting}[language=Python]
class SketchConvFunction(torch.autograd.Function):
    @staticmethod
    def forward(ctx, input, weight, bias, k, hash_idx, rand_sgn, training=True, q=2):
        '''
        Args:
            input: shape=(b, c0, w0, h0)
            weight: shape=(c1, c0*k*k)
            bias: shape=(c1)
            k: number of kernels

        Return:
            output: shape=(b, c1, w1, h1)

        Note:
            b: batch size
            c0: number of input channels
            c1: number of output channels
        '''
        b, c0, w0, h0 = input.shape
        c1 = weight.shape[0]
        w1 = w0 + 1 - k
        h1 = h0 + 1 - k

        # input tensor (b, c0, w0, h0) to patches (b*w1*h1, k*k*c0)
        fan_in = k * k * c0
        x = nn.functional.unfold(input, (k, k)).transpose(1, 2).reshape(b * w1 * h1, fan_in)

        if training:
            # sketching the input and weight matrices
            # hash_idx, rand_sgn = Sketch.rand_hashing(fan_in, q)
            x_sketch = Sketch.countsketch(x, hash_idx, rand_sgn)
            weight_sketch = Sketch.countsketch(weight, hash_idx, rand_sgn)
            z = x_sketch.matmul(weight_sketch.t())  # shape=(b*w1*h1, c1)

            # save for backprop
            shapes = torch.IntTensor([k, b, c0, w0, h0, c1, w1, h1])
            ctx.save_for_backward(x_sketch, weight_sketch, bias, shapes, hash_idx, rand_sgn)
        else:
            # the multiplication of x and w transpose
            z = x.matmul(weight.t())  # shape=(b*w1*h1, c1)

        # add bias
        bias_expand = bias.reshape(1, c1).expand([b * w1 * h1, c1])
        out_reshape = z + bias_expand  # shape=(b*w1*h1, c1)
        output = out_reshape.reshape(b, w1 * h1, c1).transpose(1, 2).reshape(b, c1, w1, h1)

        return output
        
    @staticmethod
    def backward(ctx, grad_output):
        x_sketch, weight_sketch, bias, shapes, hash_idx, rand_sgn = ctx.saved_tensors
        grad_input = grad_weight = grad_bias = None
        k, b, c0, w0, h0, c1, w1, h1 = shapes

        grad_output1 = grad_output.view(b, c1, -1).transpose(1, 2).reshape(b * w1 * h1, c1)  # shape=(b*w1*h1, c1)

        if ctx.needs_input_grad[0]:
            grad_x0 = grad_output1.matmul(weight_sketch)  # shape=(b*w1*h1, s)
            grad_x1 = Sketch.transpose_countsketch(grad_x0, hash_idx, rand_sgn)  # shape=(b*w1*h1, c0*k*k)
            grad_x2 = grad_x1.reshape(b, w1 * h1, c0 * k * k).transpose(1, 2)
            grad_input = nn.functional.fold(grad_x2, (w0, h0), (k, k))  # shape=(b, c0, w0, h0)
        if ctx.needs_input_grad[1]:
            grad_w_sketch = grad_output1.t().matmul(x_sketch)  # shape=(c1, s)
            grad_weight = Sketch.transpose_countsketch(grad_w_sketch, hash_idx, rand_sgn)  # shape=(c1, c0*k*k)
        if ctx.needs_input_grad[2]:
            grad_bias = grad_output1.sum(0)

        return grad_input, grad_weight, grad_bias, None, None, None, None, None
\end{lstlisting}

\newpage

The following PyTorch code applies CountSketch to the parameter matrices of a convolutional layer.

\begin{lstlisting}[language=Python]
class SketchConv(nn.Module):
    def __init__(self, in_channels, out_channels, kernel_size, q=2):
        super(SketchConv, self).__init__()
        self.in_channels = in_channels
        self.out_channels = out_channels
        self.kernel_size = kernel_size
        self.q = q

        self.weight = nn.Parameter(torch.Tensor(out_channels, in_channels * kernel_size * kernel_size))
        self.bias = nn.Parameter(torch.Tensor(out_channels))
        self.register_parameter('weight', self.weight)
        self.register_parameter('bias', self.bias)

        # uniform initialization
        scaling = math.sqrt(6.0)
        bound = 1 / math.sqrt(in_channels * kernel_size * kernel_size)
        self.weight.data.uniform_(-bound * scaling, bound * scaling)
        self.bias.data.uniform_(-bound, bound)

    def forward(self, input, hash_idx, rand_sgn):
        return SketchConvFunction.apply(input, self.weight, self.bias, self.kernel_size, hash_idx, rand_sgn, self.training, self.q)

    def extra_repr(self):
        return 'in_channels={}, out_channels={}, kernel_size={}, weight={}, bias={}'.format(
            self.in_channels, self.out_channels, self.kernel_size, self.weight, self.bias
        )
\end{lstlisting}